\documentclass[letterpaper]{article}
\usepackage{uai2020}
\usepackage[margin=1in]{geometry}
\usepackage{microtype}
\usepackage{graphicx}
\usepackage{subfigure}
\usepackage{color}
\usepackage{booktabs} 
\usepackage{eucal}

\usepackage{epsfig,amsmath,amssymb,amsfonts,amstext,amsthm,mathrsfs}
\usepackage{latexsym,graphics,epsf,epsfig,psfrag}
\usepackage{dsfont,color,epstopdf,fixmath}
\usepackage{enumitem}

\usepackage{algorithm,algorithmic,latexsym}

\newtheorem{assumption}{\textbf{Assumption}}

\newtheorem{corollary}{\textbf{Corollary}}
\newtheorem{Lemma}{\textbf{Lemma}}
\newtheorem{theorem}{\textbf{Theorem}}

\usepackage[acronym]{glossaries}
\makeglossaries
\newacronym{T1}{T1}{\theta_1}
\usepackage{color}

\usepackage{hyperref}

\newcommand{\mcs}{\mathcal{S}}
\newcommand{\mca}{\mathcal{A}}
\newcommand{\nn}{\nonumber}
\newcommand{\mE}{\mathbb{E}}
\newcommand{\mP}{\mathbb{P}}
\newcommand{\mg}{\mathrm \Gamma}
\newcommand{\bdelta}{\mathbf\Delta}

\usepackage{times}

\title{Finite-sample Analysis of Greedy-GQ with Linear Function Approximation under Markovian Noise}

\author{ {\bf Yue Wang} \\
Electrical Engineering \\
University at Buffalo\\
ywang294@buffalo.edu\\
\And
{\bf Shaofeng Zou} \\
Electrical Engineering \\
University at Buffalo\\
szou3@buffalo.edu
}

\begin{document}

\maketitle

\begin{abstract}
Greedy-GQ is an off-policy two timescale algorithm for optimal control in reinforcement learning \cite{maei2010toward}. This paper develops the first finite-sample analysis for the Greedy-GQ algorithm with linear function approximation under Markovian noise. Our finite-sample analysis provides theoretical justification for choosing stepsizes for this two timescale algorithm for faster convergence in practice, and suggests a trade-off between the convergence rate and the quality of the obtained policy. Our paper extends the finite-sample analyses of two timescale reinforcement learning algorithms from policy evaluation to optimal control, which is of more practical interest. Specifically, in contrast to existing finite-sample analyses for two timescale methods, e.g., GTD, GTD2 and TDC, where their objective functions are convex, the objective function of the Greedy-GQ algorithm is non-convex. Moreover, the Greedy-GQ algorithm is also not a linear two-timescale stochastic approximation algorithm.  Our techniques in this paper provide a general framework for finite-sample analysis of  non-convex value-based reinforcement learning algorithms for optimal control.


\end{abstract}
\section{Introduction}
Reinforcement learning (RL) is to find an optimal control policy to interact with a (stochastic) environment so that the accumulated reward is maximized \cite{sutton2018reinforcement}. It finds a wide range of applications in practice, e.g., robotics, computer games and recommendation systems \cite{minh2015,Minh2016,silver2016mastering,kober2013reinforcement}.

When the state and action spaces of the RL problem are finite and small, RL algorithms based on the tabular approach, which stores the action-values for each state-action pair, can be applied and usually have convergence guarantee, e.g., Q-learning \cite{watkins1992q} and SARSA \cite{Rummery1994}. However, in many RL applications, the state and action spaces are very large or even continuous. Then, the approach of function approximation can be used. Nevertheless, with function approximation in off-policy training, classical RL algorithms may diverge to infinity, e.g., Q-learning, SARSA and TD learning \cite{baird1995residual,gordon1996chattering}.

To address the non-convergence issue in off-policy training, a class of gradient temporal difference (GTD) learning algorithms were developed in \cite{maei2010toward,maei2011gradient,sutton2009fast,Sutton2009b}, including GTD, GTD2, TD with correction term (TDC), and Greedy-GQ. The basic idea is to construct squared objective functions, e.g., mean squared projected Bellman error, and then to perform stochastic gradient descent. To address the double sampling problem in gradient estimation, a weight doubling trick was proposed in \cite{sutton2009fast}, which leads to a two timescale update rule. 
One great advantage of this class of algorithms is that they can be implemented in an online and incremental fashion, which is memory and computationally efficient.

The asymptotic convergence of these two timescale algorithms has been well studied under both i.i.d.\ and non-i.i.d.\ settings \cite{sutton2009fast,Sutton2009b,maei2010toward,yu2017convergence,borkar2009stochastic,borkar2018concentration,karmakar2018two}. Furthermore, the finite-sample analyses of these algorithms are of great practical interest for algorithmic parameter tuning and design of new sample-efficient algorithms. However, these problems remain unsolved until very recently \cite{dalal2018finite,wang2017finite,liu2015finite,gupta2019finite,xu2019two}.
But, existing finite-sample analyses are only for the GTD, GTD2 and TDC algorithms, which are designed for evaluation of a given policy. The finite-sample analysis for the Greedy-GQ algorithm, which is to directly learn an optimal control policy, is still not understood and will be the focus of this paper.

In this paper, we will develop the finite-sample analysis for the Greedy-GQ algorithm with linear function approximation under Markovian noise. More specifically, we focus on the general case with a single sample trajectory and non-i.i.d.\ data. We will develop explicit bounds on the convergence of the Greedy-GQ algorithm and understand its sample complexity as a function of various parameters of the algorithm.

\subsection{Summary of Major Challenges and  Contributions}
The major challenges and our main contributions are summarized as follows.

The objective function of the Greedy-GQ algorithm is the mean squared projected Bellman error (MSPBE). Unlike the objective functions of GTD, GTD2 and TDC, which are convex, the objective function of Greedy-GQ is non-convex since the target policy is also a function of the action-value function approximation (see \eqref{eq:objective} for the objective function). In this case, the Greedy-GQ algorithm may not be able to converge to the global optimum, and existing analyses for GTD, GTD2 and TDC based on convex optimization theory cannot be directly applied. 
Moreover, the Greedy-GQ algorithm cannot be viewed as a linear two timescale  stochastic approximation due to its non-convexity, and thus existing analyses for linear two timescale stochastic approximation are not applicable. Due to the non-convexity of the objective function, convergence to the global optimum may not be guaranteed. Therefore,  we study the convergence of the gradient norm to zero (in an on-average sense, i.e., randomized stochastic gradient method \cite{ghadimi2013stochastic}), and we focus on convergence to stationary points. In this paper, we develop a novel methodology for finite-sample analysis of the Greedy-GQ algorithm, which solves reinforcement learning problems from a non-convex optimization perspective. This may be of independent interest for a wide range of reinforcement learning problems with non-convex objective functions.


In this paper, we focus on the most general scenario where there is a single sample trajectory and the data are non-i.i.d.. This non-i.i.d.\ setting will invalidate the martingale noise assumption commonly used in stochastic approximation (SA) analysis \cite{maei2010toward,dalal2018finite,borkar2018concentration}. Our approach is to analyze RL algorithms from a non-convex optimization perspective, and does not require the martingale noise assumption. Thus, our approach has a much broader applicability. 

Moreover, the propagation of the stochastic bias in the gradient estimate caused by the Markovian noise in the two timescale updates makes the analysis even more challenging. We develop a comprehensive characterization of the stochastic bias and establish the convergence rate of the Greedy-GQ algorithm under constant stepsizes. More importantly, we develop a novel recursive approach of bounding the bias caused by the tracking error, i.e., the error in the fast timescale update. Specifically, our approach is to recursively plug the obtained bound back into the analysis to tighten the final bound on the bias.

We show that under constant stepsizes, i.e., $\alpha_t=\frac{1}{T^a}$ and $\beta_t=\frac{1}{T^b}$ for $0\leq t\leq T$, the Greedy-GQ algorithm converges as fast as $\mathcal O\left(\frac{1}{T^{1-a}}+\frac{\log T}{T^{\min\{b,a-b\}}}\right)$. We also derive the best choice of $a$ and $b$ so that the above rate is the fastest. Specifically, when $a=\frac{2}{3}$ and $b=\frac{1}{3}$, the Greedy-GQ algorithm converges as fast as $\mathcal O\left(\frac{\log T}{T^{\frac{1}{3}}}\right)$.
We further characterize the trade-off between the convergence speed and the quality of the obtained policy. Specifically, the algorithm needs more samples to converge if the target policy is more ``greedy", e.g., a larger parameter $\sigma$ in softmax makes the policy more ``greedy", and will require more samples to converge. Our experiments also validate this theoretical observation.
\subsection{Related Work}
In this subsection, we provide an overview of closely related work. Specifically, we here focus on value-based RL algorithms with function approximation. We note that there are many other types of approaches, e.g.,  policy gradient and fitted value/policy iteration, which are not discussed in this paper.

 \textbf{TD, Q-learning and SARSA with function approximation.} TD with linear function approximation was shown to converge asymptotically in \cite{Tsitsiklis1997}, and its finite-sample analysis was established in \cite{Dalal2018a,Laksh2018,bhandari2018finite,srikant2019} under both i.i.d. and non-i.i.d.\ settings. Moreover, the finite-sample analysis of TD with over--parameterized neural function approximation was developed in \cite{cai2019neural}. Q-learning and SARSA with linear function approximation were shown to converge asymptotically under certain conditions \cite{melo2008analysis,perkins2003convergent} and their finite-sample analyses were developed in \cite{zou2019finite,chen2019performance}. However, these algorithms may diverge under off-policy training.  Different from TD, Q-learning and SARSA, the Greedy-GQ algorithm follows a stochastic gradient descent type update. However, the updates of TD, Q-learning and SARSA do not exactly follow a gradient descent type, since the ``gradient" therein is not  gradient of any function \cite{maei2010toward}. Moreover, the Greedy-GQ algorithm is a two timescale one, and thus requires more involved analysis than these one timescale methods.
    
 \textbf{GTD algorithms.} The GTD, GTD2 and TDC algorithms were shown to converge asymptotically in \cite{Sutton2009b,sutton2009fast,yu2017convergence}. Their finite-sample analyses were further developed recently in \cite{dalal2018finite,wang2017finite,liu2015finite,gupta2019finite,xu2019two} under i.i.d.\ and non-i.i.d.\ settings. The Greedy-GQ algorithm studied in this paper is fundamentally different from the above three algorithms. This is due to the fact that the Greedy-GQ algorithm is for optimal control and its objective function is non-convex; whereas the GTD, GTD2 and TDC algorithms are for policy evaluation, and their objective functions are convex. Therefore, new techniques need to be developed to tackle the non-convexity for the finite-sample analysis for Greedy-GQ. Moreover, general linear two timescale stochastic approximation has also been studied. Although the Greedy-GQ algorithm follows a two timescale update rule, but it is not linear. Furthermore, the general non-linear two timescale stochastic approximation was studied in \cite{borkar2018concentration}. However, the Greedy-GQ algorithm under Markovian noise does not satisfy the martingale noise assumption therein. Moreover, our paper uses a non-convex optimization based approach to develop the finite-sample analysis, which is different from the approach used in \cite{borkar2018concentration}.

\section{Preliminaries}\label{sec:pre}
\subsection{Markov Decision Process}
In RL problems, a Markov Decision Process (MDP) is usually used to model the interaction between an agent and a stochastic environment. Specifically, an MDP consists of $(\mathcal{S},\mathcal{A},  \mathsf{P}, r, \gamma)$, where $\mathcal{S}\subset \mathbb R^d$ is the state space, $\mathcal{A}$ is a finite set of actions, and $\gamma\in(0,1)$ is the discount factor. Denote the state at time $t$ by  $S_t$, and the action taken at time $t$ by $A_t$. Then the measure $\mathsf P$ denotes the action-dependent transition kernel of the MDP: 
\begin{flalign}
\mathbb{P}(S_{t+1}\in U|S_t=s,A_t=a)=\int_U\mathsf{P}(dx|s,a),
\end{flalign}
for any measurable set $U\subseteq \mathcal S$. The reward at time $t$ is given by $r_t=r(S_t,A_t,S_{t+1})$, which is the reward of taking action $A_t$ at state $S_t$ and transitioning to a new state $S_{t+1}$. Here $r:\mcs\times\mca\times\mcs\to\mathbb R$ is the reward function, and is assumed to be uniformly bounded, i.e., 
\begin{align}
    0\leq r(s,a,s')\leq r_{\max},  \forall (s,a,s')\in \mcs\times\mca\times\mcs.
\end{align}

A stationary policy maps a state $s\in\mcs$ to a probability distribution $\pi(\cdot|s)$ over $\mca$, which does not depend on time. For a policy $\pi$, its value function $V^\pi: \mcs\to\mathbb R$ is defined as the expected accumulated discounted reward by executing the policy $\pi$ to obtain actions:	
\begin{align}
    V^\pi\left(s_0\right)=\mE\left[\sum_{t=0}^{\infty}\gamma^t   r(S_t,A_t,S_{t+1})|S_0=s_0\right].
\end{align}
The action-value function $Q^\pi:\mcs\times\mca\rightarrow\mathbb R$ of policy $\pi$ is defined as
\begin{align}
	Q^\pi(s,a)=\mE_{S'\sim \mathsf P(\cdot|s,a)}\left[r(s,a,S')+\gamma V^\pi(S')\right].
\end{align}
The goal of optimal control in RL is to find the optimal policy $\pi^*$ that maximizes the value function for any initial state, i.e., to solve the following problem:
\begin{flalign}
V^*(s)=\sup_\pi V^\pi(s), \,\forall s\in\mcs.
\end{flalign}
	We can also define the optimal action-value function as 
\begin{flalign}
	Q^*(s,a)=\sup_\pi Q^\pi(s,a), \,\forall (s,a)\in \mcs\times\mca.
\end{flalign}
Then, the optimal policy $\pi^*$ is greedy w.r.t. $Q^*$. 
	The Bellman operator $\mathbf T$ is defined as
\begin{align}
    (\mathbf TQ)(s,a)=&\int_\mcs (r(s,a,s')\nn\\
    &+\gamma \max_{b\in\mca}Q(s',b))\mathsf{P}(ds'|s,a).
\end{align}
	It is clear that $\mathbf T$ is contraction in the sup norm defined as  $\|Q\|_{\sup}=\sup_{(s,a)\in\mcs\times\mca}|Q(s,a)|$, and the optimal action-value function $Q^*$ is the fixed point of $\mathbf T$ \cite{bertsekas2011dynamic}.

\subsection{Linear Function Approximation}
In many modern RL applications, the state space is usually very large or even continuous. Therefore, classical tabular approach cannot be directly applied due to memory and computational constraint \cite{sutton2018reinforcement}. In this case, the approach of function approximation can be applied, which uses a family of parameterized function to approximate the action-value function. In this paper, we focus on linear function approximation.

Consider a set of $N$ fixed base functions $\phi^{(i)}$: $\mcs\times\mca\rightarrow \mathbb R,\, i=1,\ldots,N$. Further consider a family of real-valued functions $\mathcal Q=\{Q_\theta:\theta\in\mathbb R^N\}$ defined on $\mcs\times\mca$, which consists of linear combinations of $\phi^{(i)}$, $i=1,\ldots,N$. Specifically, 
\begin{flalign}
Q_\theta(s,a)=\sum_{i=1}^N \theta(i)\phi^{(i)}_{s,a}=\phi_{s,a}^\top \theta.
\end{flalign}
The goal is to find a $Q_\theta$ with a compact representation in $\theta$ to approximate the optimal action-value function $Q^*$.






\subsection{Greedy-GQ Algorithm}
In this subsection, we introduce the Greedy-GQ algorithm, which was originally proposed in \cite{maei2010toward} to solve the problem of optimal control in RL under off-policy training.

For the Greedy-GQ algorithm, a fixed behavior policy $\pi_b$ is used to collect samples. It is assumed that the Markov chain $\{X_t,A_t\}_{t=0}^\infty$ induced by the behavior policy $\pi_b$ and the Markov transition kernel $\mathsf P$ is uniformly ergodic with the invariant measure denoted by $\mu$.

The main idea of the Greedy-GQ algorithm is to design an objective function, and further to employ a stochastic gradient descent optimization approach together with a weight doubling trick (a two timescale update) \cite{Sutton2009b} to minimize the objective function. Specifically, the goal is to minimize the following mean squared projected Bellman error (MSPBE):
\begin{flalign}\label{eq:objective}
J(\theta)\triangleq||\mathbf \Pi\mathbf T^{\pi_{\theta}}Q_{\theta}-Q_{\theta}||_{\mu}.
\end{flalign}
Here $\|Q(\cdot,\cdot)\|_\mu\triangleq\int_{s\in\mcs,a\in\mca}d\mu_{s,a}Q(s,a)$; $\mathbf T^{\pi}$ is the Bellman operator: 
\begin{align}
    \mathbf T^{\pi}Q(s,a)\triangleq\mE_{S',A'}[r(s,a,S')+\gamma Q(S',A'))],
\end{align}
where $S'\sim \mathsf P(\cdot|s,a)$, and $A'\sim \pi(\cdot|S')$; 
$\mathbf \Pi$ is a projection operator which projects an action-value function to the function space $\mathcal{Q}$ with respect to $||\cdot||_{\mu}$, i.e.,
$\mathbf \Pi \hat Q=\arg\min_{Q\in \mathcal Q}\|Q-\hat Q\|_\mu$; and $\pi_\theta$ is a stationary policy, which is a function of $\theta$. 





We note that the objective function in \eqref{eq:objective} is non-convex since the parameter $\theta$ is also in the Bellman operator, i.e., $\pi_\theta$. Moreover, unlike GTD, GTD2 and TDC, the objective function of the Greedy-GQ algorithm is not a quadratic function of $\theta$. Thus, the Greedy-GQ algorithm is not a linear two timescale stochastic approximation algorithm.


Define $\delta_{s,a,s'}(\theta)=r({s,a,s'})+\gamma\Bar{V}_{s'}(\theta)-\theta^\top  \phi_{s,a}$, and $\Bar{V}_{s'}(\theta)=\sum_{a'} \pi_{\theta}(a'|s')\theta^\top  \phi_{s',a'}$. In this way, the objective function in \eqref{eq:objective} can be rewritten equivalently as follows
\begin{align}
     J({\theta})=&\mathbb{E}_{\mu}[\delta_{S,A,S'}({\theta})\phi_{S,A}]^\top   \mathbb{E}_{\mu}[\phi_{S,A}\phi_{S,A}^\top  ]^{-1}\nn\\
     &\times\mathbb{E}_{\mu}[\delta_{S,A,S'}({\theta})\phi_{S,A}],
\end{align}
where $(S,A)\sim\mu$, and $S'\sim \mathsf P(\cdot|S,A)$ is the subsequent state.



To compute a gradient to $J(\theta)$, we will need to compute the gradient to $\delta_{S,A,S'}(\theta)$, and thus the gradient to $\Bar{V}_{S'}(\theta)$. 
Suppose  $\hat{\phi}_{S'}(\theta)$ is an unbiased estimate of the gradient to $\Bar{V}_{S'}(\theta)$ given $S'$, then $\psi_{S,A,S'}(\theta)=\gamma\hat\phi_{S'}(\theta)-\phi_{S,A}$ is a gradient of $\delta_{S,A,S'}(\theta)$. 
Then, the gradient to $J(\theta)/2$ can be computed as follows:
\begin{align}\label{eq:5}
    &\mathbb{E}_{\mu}[\psi_{S,A,S'}(\theta)\phi_{S,A}^\top  ]\mathbb{E}_{\mu}[\phi_{S,A}\phi_{S,A}^\top  ]^{-1}\mathbb{E}_{\mu}[\delta_{S,A,S'}(\theta)\phi_{S,A}] \nonumber\\
    &=-\mathbb{E}_{\mu}[\delta_{S,A,S'}(\theta)\phi_{S,A}]+\gamma\mathbb{E}_{\mu}[\hat{\phi}_{S'}(\theta)\phi_{S,A}^\top  ]\omega^*(\theta),
\end{align}
where $\omega^*(\theta)=\mathbb{E}_{\mu}[\phi_{S,A}\phi_{S,A}^\top  ]^{-1} \mathbb{E}_{\mu}[\delta_{S,A,S'}({\theta})\phi_{S,A}].$
To get an unbiased estimate of \eqref{eq:5}, two independent samples of $(S,A,S')$ are needed, which is not applicable when there is a single sample trajectory. Then, a weight doubling trick \cite{Sutton2009b} was used in \cite{maei2010toward} to construct the Greedy-GQ algorithm with the following updates (see Algorithm \ref{al:1} for more details):
\begin{align}
    &\theta_{t+1}=\theta_t+\alpha_t(\delta_{t+1}(\theta_t)\phi_t-\gamma(\omega_t^\top  \phi_t)\hat{\phi}_{t+1}(\theta_t)),\\
    &\omega_{t+1}=\omega_t+\beta_t(\delta_{t+1}(\theta_t)-\phi_t^\top  \omega_t)\phi_t,\label{eq:omegaupdate}
\end{align}
where $\alpha_t>0$ and $\beta_t>0$ are non-increasing stepsizes, $\delta_{t+1}(\theta)\triangleq\delta_{s_t,a_t,s_{t+1}}(\theta)$ and $\phi_t\triangleq\phi_{s_t,a_t}$. 
For more details of the derivation of the Greedy-GQ algorithm, we refer the readers to \cite{maei2010toward}.

\begin{algorithm}[!htb]
		\caption{Greedy-GQ \cite{maei2010toward}}\label{al:1}
		\begin{algorithmic}
			\STATE \textbf{Initialization:}
			\STATE $\theta_0$, $\omega_0$, $s_0$, $\phi^{(i)}$, for $i=1,2,...,N$
			\STATE \textbf{Method:}
			\STATE $\pi_{\theta_0}\leftarrow\mg(\phi^\top  \theta_0)$
			\FOR {$t=0,1,2,...$}
			\STATE {Choose $a_t$ according to $\pi_b(\cdot|s_t)$}
			\STATE Observe $s_{t+1}$ and $r_{t}$
			\STATE $\Bar{V}_{s_{t+1}}(\theta_{t}) \leftarrow \sum_{a'\in \mathcal{A}} \pi_{\theta_{t}}(a' |s_{t+1})\theta_{t}^\top  \phi_{s_{t+1},a'}$
			\STATE $\delta_{t+1}(\theta_{t})\leftarrow r_{t}+\gamma\Bar{V}_{s_{t+1}}(\theta_{t})-\theta_{t}^\top  \phi_{t} $
			\STATE $\hat{\phi}_{t+1}(\theta_{t})\leftarrow$ gradient of $\Bar{V}_{s_{t+1}}(\theta_{t})$
			\STATE $\theta_{t+1} \leftarrow  \theta_{t}+\alpha_{t}(\delta_{t+1}(\theta_{t})\phi_{t}-\gamma(\omega_{t}^\top  \phi_{t})\hat{\phi}_{t+1}(\theta_{t}))$
			\STATE $\omega_{t+1} \leftarrow \omega_{t}+\beta_{t}(\delta_{t+1}(\theta_{t})-\phi_{t}^\top  \omega_{t})\phi_{t}$
			\STATE  \textbf{Policy improvement}: $\pi_{\theta_{t+1}}\leftarrow\mg(\phi^\top  \theta_{t+1})$
			\ENDFOR
		\end{algorithmic}
	\end{algorithm}
In Algorithm \ref{al:1}, $\mg$ is a policy improvement operator, which maps an action-value function to a policy, e.g., greedy, $\epsilon$-greedy, and softmax and mellowmax \cite{Asadi2016}. 

\section{Finite-Sample Analysis for Greedy-GQ}
In this section, we will first introduce some technical assumptions, and then present our main results.

We make the following standard assumptions.
\begin{assumption}[Problem solvability]
 The matrix $C=\mathbb{E}_{\mu}[\phi_t\phi_t^\top  ]$ is non-singular.
\end{assumption} 
\begin{assumption}[Bounded feature]
  $\|\phi_{s,a}\|_2\leq 1, \forall (s,a)\in\mcs\times\mca$.
\end{assumption}




\begin{assumption}[Geometric uniform ergodicity]\label{ass:1}
 There exists some constants $m>0$ and $\rho \in (0,1)$ such that 
\begin{align}
    \sup_{s\in\mathcal{S}} d_{TV}(\mathbb{P}(s_t |s_0=s), \mu) \leq m\rho^t ,
\end{align}
for any $t>0$, where $d_{TV}$ is the total-variation distance between the probability measures.
\end{assumption}


In this paper, we focus on policies that are smooth. Specifically,  $\pi_{\theta}(a|s)$ and $\nabla \pi_{\theta} (a|s)$ are Lipschitz functions of $\theta$.
\begin{assumption}[Policy smoothness]\label{assump:policy}
 The policy $\pi_\theta(a|s)$ is $k_1$-Lipschitz and $k_2$-smooth, i.e., for any $(s,a) \in \mcs\times\mca$,
\begin{align}
    \|\nabla \pi_{\theta}(a|s)\|\leq k_1, \forall \theta,
\end{align}and,
\begin{align}
    \|\nabla\pi_{\theta_1}(a|s)-\nabla\pi_{\theta_2}(a|s)\| \leq k_2 \| \theta_1-\theta_2\|, \forall \theta_1,\theta_2
    .
\end{align} 
\end{assumption}
We note that the smaller the $k_1$ and $k_2$ are, the smoother the policy is.
This family contains many policies as special cases, e.g., softmax and mellowmax \cite{Asadi2016}. We also note that the greedy policy is not smooth, since it is not differentiable. 

To justify the feasibility of Assumption \ref{assump:policy} in practice, in the following, we first provide an example of the softmax policy, and show that it is Lipschitz and smooth in $\theta$.
Consider the softmax operator, where for any $(a,s)\in \mca\times\mcs$ and $\theta\in \mathbb R^N$,
\begin{align}\label{eq:softmax}
    \pi_{\theta}(a|s)=\frac{e^{\sigma {\theta}^\top \phi_{s,a}}}{\sum_{a' \in \mathcal{A}}e^{\sigma {\theta}^\top  \phi_{s,a'}}},
\end{align} 
for some $\sigma>0$. 
\begin{Lemma}\label{lemma:softmax_smooth}
The softmax policy $\pi_{\theta}(a|s)$ is $2\sigma$-Lipschitz and $8\sigma^2$-smooth, i.e., for any $(s,a)\in\mcs\times\mca$, and for any $\theta_1,\theta_2\in\mathbb R^N$, 
\begin{align}
|\pi_{\theta_1}(a|s)-\pi_{\theta_2}(a|s)| &\leq 2\sigma \|\theta_1-\theta_2 \|,\\
\|\nabla\pi_{\theta_1}(a|s)-\nabla \pi_{\theta_2}(a|s)  \|&\leq 8\sigma^2 \|\theta_1-\theta_2 \|.
\end{align}
\end{Lemma}
As $\sigma\rightarrow\infty$, the softmax policy approximates the greedy policy asymptotically, however its Lipschitz and smoothness constants also go to infinity. 


It can be seen from \eqref{eq:objective} that the objective function of the Greedy-GQ algorithm is non-convex. It may not be possible to guarantee the convergence of the algorithm to the global optimum. Therefore, to measure the convergence rate, we consider the convergence rate of the gradient norm to zero. Furthermore, motivated by the randomized stochastic gradient method in \cite{ghadimi2013stochastic}, which is designed to analyze non-convex optimization problems, in this paper, we also consider a randomized version of the Greedy-GQ algorithm in Algorithm \ref{al:1}. Specifically, let $M$  be an independent random variable with probability mass function $\mP_M$. For steps from 1 to $M$, call the Greedy-GQ algorithm in Algorithm \ref{al:1}. The final output is then $\theta_M$.

In the following theorem, we provide the convergence rate bound for $\mathbb{E}[\|\nabla  J(\theta_M)\|^2]$ when constant stepsizes are used. Specifically, let $M \in \left\{1,2,...T\right\}$ and
\begin{flalign}\label{eq:M}
\mathbb{P}(M=k)=\frac{\alpha_k}{\sum^T_{t=0}\alpha_t}.
\end{flalign}
\begin{theorem}\label{thm:main}
Consider the following stepsizes: $\beta=\beta_t=\frac{1}{T^b}$, and $\alpha=\alpha_t=\frac{1}{T^a}$, where $\frac{1}{2}< a\leq 1$ and $0<b\leq a$. Then we have that for $T>0$,
\begin{align}\label{eq:theorembound}
    \mathbb{E}[\|\nabla  J(\theta_M)\|^2]=\mathcal{O}\left(\frac{1}{T^{1-a}}+\frac{\log T}{T^{\min\{b,a-b\}}}\right).
\end{align}

\end{theorem}
Here we only provide the order of the bound in terms of $T$. An explicit bound can also be derived, which however is cumbersome and tedious.  To understand how different parameters, e.g., $L, C, m,\rho$, affect the convergence speed, we refer the readers to equation \eqref{eq:80}  in the appendix.

Although it is not explicitly characterized in \eqref{eq:theorembound}, we note that as $k_1$ and $k_2$ increases, the bound will become looser and thus the algorithm will need more samples to converge. For a more ``greedy" target policy with  larger $k_1$ and $k_2$, it will require more samples to converge. This suggests a practical trade-off between the quality of the obtained policy and the sample complexity.

Theorem \ref{thm:main} characterizes the relationship between the convergence rate and the choice of the stepsizes $\alpha_t$ and $\beta_t$. We further optimize over the choice of the stepsizes and obtain the best bound as in the following corollary.
\begin{corollary}\label{col:1}
If we choose $a=\frac{2}{3}$ and $b=\frac{1}{3}$, then the best rate of the bound in  \eqref{eq:theorembound} is obtained as follows:
\begin{flalign}
\mathbb{E}[\|\nabla  J(\theta_M)\|^2]=\mathcal{O}\left(\frac{\log T}{T^{\frac{1}{3}}}\right).
\end{flalign}
\end{corollary}

For the general non-convex optimization problem with a Lipschitz gradient, the convergence rate of the randomized stochastic gradient method is $\mathcal O(T^{-\frac{1}{2}})$ \cite{ghadimi2013stochastic}. However, the gradient estimate in that problem is unbiased, and the update is one timescale. In our problem, we have a two timescale update rule. Although the fast timescale updates much faster than the slow timescale, there still exists an estimation error, which we call it ``tracking error". Specifically, the tracking error is defined as 
\begin{flalign}
z_t=w_t-w^*(\theta_t).
\end{flalign}
Moreover, in this paper, we consider the practical scenario where a single sample trajectory with Markovian noise is used. Therefore, for the Greedy-GQ algorithm, there exists bias in the gradient estimate, which justifies the difference in the convergence rate from the one for general non-convex optimization problems \cite{ghadimi2013stochastic}.

\section{Proof Sketch}
In this section, we provide an outline of the proof, and highlight our major technical contributions. For a complete proof, we refer the readers to the appendix.

The proof can summarized in the following five steps.
\begin{enumerate}
    \item We first prove that $J(\theta)$ is Lipschitz and smooth.
    \item We then decompose the error recursively.
    \item We provide a comprehensive characterization of stochastic bias terms and the tracking error in the two timescale updates. 
    \item We then recursively plug the obtained bound on $\mathbb{E}[\|\nabla  J(\theta_M)\|^2]$ back into the analysis, and repeat recursively to obtain the tightest bound. 
    \item We then optimize the convergence rate over the choice of stepsizes.
\end{enumerate}
In the following, we discuss the proof sketch step by step with more details.

\textbf{Step 1.} 
We first provide a characterization of the geometric property of the objective function $J(\theta)$. Specifically, we show that if $\pi_\theta$ is Lipschitz and smooth (satisfying Assumption \ref{assump:policy}), then $J(\theta)$ is also Lipschitz and $K$-smooth for some $K>0$, i.e., for any $\theta_1$ and $\theta_2$,
\begin{flalign}
\|\nabla J(\theta_1)-\nabla J(\theta_2)\| \leq K||\theta_1-\theta_2||.
\end{flalign} 
Here,  larger $k_1$ and $k_2$ imply a larger $K$. As will be seen later in Step 2 and Step 3, a larger $K$ means a looser bound and a higher sample complexity. This theoretical assertion will also be
validated in our numerical experiments.

Recall that $J({\theta})$ can be equivalently written as  
$
\mathbb{E}_{\mu}[\delta_{S,A,S'}({\theta})\phi_{S,A}]^\top   \mathbb{E}_{\mu}[\phi_{S,A}\phi_{S,A}^\top  ]^{-1}\mathbb{E}_{\mu}[\delta_{S,A,S'}({\theta})\phi_{S,A}],
$
which has a quadratic form in $\mathbb{E}_{\mu}[\delta_{S,A,S'}({\theta})\phi_{S,A}]$. Therefore, it suffices to show that $\mathbb{E}_{\mu}[\delta_{S,A,S'}({\theta})\phi_{S,A}]$ is bounded, Lipschitz and smooth, which is clear from its definition and the fact that $\pi_\theta$ is Lipschitz and smooth. 

%




\textbf{Step 2.}
Since the object function $J(\theta)$ is Lipschitz and $K$-smooth, then by Taylor expansion, we have that
\begin{align}\label{eq:1J}
    &J(\theta_{t+1})-J(\theta_t)- \langle \theta_{t+1}-\theta_t, \nabla J(\theta_t) \rangle \nn\\
    &\leq \frac{K}{2} \|\theta_{t+1}-\theta_t \|^2.
\end{align}
Denote by $G_{t+1}(\theta,\omega)=(\delta_{t+1}(\theta)\phi_t-\gamma(\omega^\top  \phi_t)\hat{\phi}_{t+1}(\theta))$. Then, the difference between $\theta_t$ and $\theta_{t+1}$ is $\alpha_t G_{t+1}(\theta_t,\omega_t)$. The inequality \eqref{eq:1J} can be further written as
\begin{align}\label{eq:2}
     &J(\theta_{t+1})-J(\theta_t)- \alpha_t\langle G_{t+1}(\theta_t,\omega_t), \nabla J(\theta_t) \rangle\nn\\
     &\leq \frac{K\alpha_t^2}{2} \|G_{t+1}(\theta_t,\omega_t) \|^2.
\end{align}
 
Note that $G_{t+1}(\theta_t,\omega_t)$ is the stochastic gradient used in the Greedy-GQ algorithm. Due to the two timescale update and the Markovian noise, the stochastic gradient is biased. For a finite-sample analysis, we will then need to characterize the stochastic bias in the gradient estimate $G_{t+1}(\theta_t,\omega_t)$ explicitly. 
 
We first consider the difference between the true gradient $\nabla  J(\theta_t)$ and the gradient estimate $G_{t+1}(\theta_t,\omega_t)$ used in the Greedy-GQ algorithm, which is denoted by   $\bdelta_t=-2G_{t+1}(\theta_t, \omega_t)-\nabla  J(\theta_t)$. Plug this in the inequality \eqref{eq:2}, and we obtain that
\begin{align}\label{eq:3}
    &J(\theta_{t+1})-J(\theta_t)+\frac{\alpha_t}{2}\left\langle (\bdelta_t+\nabla J(\theta_t)), \nabla J(\theta_t) \right\rangle\nn\\
    &=J(\theta_{t+1})-J(\theta_t)+\frac{\alpha_t}{2}\|\nabla J(\theta_t) \|^2\nn\\
    &\quad+ \alpha_t\left\langle \frac{1}{2}\bdelta_t, \nabla J(\theta_t) \right\rangle\nn\\
    &\leq \alpha_t^2\frac{K}{2} \|G_{t+1}(\theta_t,\omega_t) \|^2.
\end{align}
  

Recall the definition of the random variable $M$ in \eqref{eq:M}. Applying \eqref{eq:3} recursively, we have that
\begin{flalign}\label{eq:4}
&\mathbb{E}[\|\nabla  J(\theta_M)\|^2] \nn\\
&\leq \frac{1}{\sum_{t=0}^T\alpha_t}\Bigg((J(\theta_0)-J(\theta_{T+1}))\nn\\
&\hspace{1cm}+ \frac{K}{2}\sum^T_{t=0}\alpha_t^2 \mathbb{E}[\| G_{t+1}(\theta_t,\omega_t)\|^2]\nn\\
    &\hspace{1cm}-\sum^T_{t=0} \frac{\alpha_t}{2}\left\langle \bdelta_t, \nabla J(\theta_t) \right\rangle\Bigg).
\end{flalign}
From \eqref{eq:4}, it can be seen that to understand the convergence rate of $\mathbb{E}[\|\nabla  J(\theta_M)\|^2]$, we need to bound the three terms on the right hand side of \eqref{eq:4}.
The first and second terms are straightforward to bound since $J(\theta)$ is non-negative for any $\theta$, and $\|G_{t+1}\|$ is uniformly bounded by some constant. 

For the third term $\left\langle \bdelta_t, \nabla J(\theta_t) \right\rangle$, it can be further decomposed into the following two parts
\begin{flalign}\label{eq:5a}
&\big\langle \nabla J(\theta_t),-2G_{t+1}(\theta_t, \omega_t)+2G_{t+1}(\theta_t, \omega^*(\theta_t)) \big\rangle\nn\\
    &- \big\langle \nabla J(\theta_t), \nabla J(\theta_t)+2G_{t+1}(\theta_t, \omega^*(\theta_t)) \big\rangle,
\end{flalign}
where the first part is corresponding to the tracking error, and the second part is corresponding to the stochastic bias caused by the Markovian noise.

\textbf{Step 3.} We then provide bounds for each term in \eqref{eq:4} and \eqref{eq:5a}. 
For the first and second terms in \eqref{eq:4}, it is straightforward to develop their upper bounds. For the first term in \eqref{eq:5a}, it can be upper bounded by exploiting the Lipschitz property of $G_{t+1}(\theta,\omega)$ in $\omega$. Specifically,
\begin{flalign}\label{eq:31a}
&\big\langle \nabla J(\theta_t),-2G_{t+1}(\theta_t, \omega_t)+2G_{t+1}(\theta_t, \omega^*(\theta_t)) \big\rangle\nn\\
&\leq  \xi_1 \|\nabla J(\theta_t)\| \|\omega_t-\omega^*(\theta_t)\|,
\end{flalign}
for some $\xi_1>0$. Thus, it suffices to bound the tracking error $\|\omega_t-\omega^*(\theta_t)\|$. The bound on the tracking error is difficult due to the complicated coupling between the parameter $\omega_t$, $\theta_t$ and the sample trajectory. We decouple such the dependence between $\omega_t$, $\theta_t$ and the samples by looking $\tau$ steps back, where $\tau$ is the mixing time of the MDP. By the geometric uniform ergodicity, conditioning on $\omega_{t-\tau}$ and $\theta_{t-\tau}$, the distribution of $(s_t,a_t)$ is close to the stationary distribution $\mu$. Thus, the expectation of the tracking error can be bounded. 


We then bound the second term in \eqref{eq:5a}. We know that for any fixed $\theta$,
$\mE_\mu[\nabla J(\theta)+2G_{t+1}(\theta, \omega^*(\theta))]=0.$ However, $\theta_t$ and $S_t,A_t,S_{t+1}$ are not independent. Similarly, we exploit the geometric uniform ergodicity of the MDP. For simplicity, we denote by 
\begin{flalign}
\zeta(\theta_t,O_t)=\big\langle \nabla J(\theta_t), \nabla J(\theta_t)+2G_{t+1}(\theta_t, \omega^*(\theta_t)) \big\rangle,
\end{flalign}
where $O_t=\{S_t,A_t,S_{t+1},r_t\}$.
We can show that $\zeta(\theta,O_t)$ is Lipschitz in $\theta$. Thus, if we look $\tau$ step back, then 
\begin{flalign}
|\zeta(\theta_t,O_t)-\zeta(\theta_{t-\tau},O_t)|\leq c_\zeta\|\theta_t-\theta_{t-\tau}\|,
\end{flalign} for some $c_\zeta>0$.
Therefore, 
\begin{flalign}
\zeta(\theta_t,O_t)\leq \zeta(\theta_{t-\tau},O_t) +c_\zeta\|\theta_t-\theta_{t-\tau}\|.
\end{flalign}
Since we are using small stepsizes, then $\|\theta_t-\theta_{t-\tau}\|$ should be small. In other words, the difference between $\zeta(\theta_t,O_t)$ and $\zeta(\theta_t,O_t)$ is small. By the geometric uniform ergodicity, for any $\theta_{t-\tau}$, the distribution of $O_t$ is close to the stationary distribution $\mu$. Thus, even $\theta_{t-\tau}$ and $O_t$ are not independent, we can still upper bound $\mE[\zeta(\theta_{t-\tau},O_t)]$. In this way, we decouple the dependence between $\theta_t$ and $O_t$, and we can obtain the bound on the gradient bias.





\textbf{Step 4.}
After Step 3, we can obtain the following bound on $\mathbb{E}[\|\nabla  J(\theta_M)\|^2]$:
\begin{flalign}\label{eq:34a}
\mathbb{E}[\|\nabla  J(\theta_M)\|^2]=\mathcal O\left(\frac{1}{T^{1-a}}+\frac{\sqrt{\log T}}{T^{\frac{1}{2}\min\{b,a-b\}}}\right).
\end{flalign}
This bound is obtained by upper bounding $\|\nabla J(\theta_t)\|$ on the right hand side of \eqref{eq:4} using a constant. Obviously, $\mathbb{E}[\|\nabla  J(\theta_M)\|^2]\to 0$ as $T\to \infty$, and thus using a constant to upper bound $\nabla J(\theta_t)$ is not tight. 

In this step, we recursively use the obtained bound to further tighten the bound on $\mathbb{E}[\|\nabla  J(\theta_M)\|^2]$. Specifically, we plug \eqref{eq:34a} back into \eqref{eq:31a} in Step 3. If $1-a> \min\{b,a-b\}$, then the second term on the right hand side of  \eqref{eq:34a} dominates. Plugging \eqref{eq:34a} back into \eqref{eq:31a} will further tighten the bound to the following one:
\begin{flalign}
\mathbb{E}[\|\nabla  J(\theta_M)\|^2]=\mathcal O\left(\frac{1}{T^{1-a}}+\frac{\log^{\frac{3}{4}} T}{T^{\frac{3}{4}\min\{b,a-b\}}}\right).
\end{flalign}
Repeat this procedure, we can then obtain the following bound:
\begin{flalign}\label{eq:36a}
\mathbb{E}[\|\nabla  J(\theta_M)\|^2]=\mathcal O\left(\frac{1}{T^{1-a}}+\frac{\log T}{T^{\min\{b,a-b\}}}\right).
\end{flalign}
If $1-a\leq  \frac{1}{2}\min\{b,a-b\}$, then the first term in \eqref{eq:34a} dominates. Therefore, the above recursive refinement will not improve the convergence rate. 
If $\frac{1}{2}\min\{b,a-b\}\leq 1-a\leq  \min\{b,a-b\}$, we can apply our recursive bounding trick finite times until the first term $\mathcal O\left(\frac{1}{T^{1-a}} \right)$ in \eqref{eq:34a} dominates.
Combining the analyses for the three cases, the overall convergence rate bound can be obtained, which is as in \eqref{eq:36a}.

\textbf{Step 5.} Given the convergence rate bound in \eqref{eq:36a}, in this step, we optimize over the choice of the stepsizes to obtain the fastest convergence rate. Recall that $\frac{1}{2}< a\leq 1$ and $0<b\leq a$. Then, it can be derived that when $a=\frac{2}{3}$ and $b=\frac{1}{3}$, the best convergence rate that is achievable  in \eqref{eq:36a} is $\mathcal O\left(\frac{\log T}{T^{\frac{1}{3}}}\right)$.

\section{Numerical Experiments}
In this section, we present our numerical experiments. Specifically, we investigate how the Lipschitz and smoothness constants affect the convergence of the Greedy-GQ algorithm. 
We use the the softmax operator as an example. Recall that in Lemma \ref{lemma:softmax_smooth}, the Lipschitz and smoothness constants of the softmax operator is an increasing function of $\sigma$ in \eqref{eq:softmax}.


As has been observed in our finite-sample analysis, the upper bound on the gradient norm increases with $K$, and thus increases with $\sigma$. This suggests a higher sample complexity as the target policy becomes more ``greedy". We will numerically validate this observation by simulating the Greedy-GQ algorithm for different values of $\sigma$ in \eqref{eq:softmax}. 


We consider a simple example: $\mathcal{S}=\{1,2,3,4\}$ and $\mathcal{A}=\{1,2\}$. For the first MDP we consider, taking any action at any state will have the same probability to transit to any state, i.e. $\mathbb{P}(s'|s,a)=\frac{1}{4}$ for any $(s,a,s')$. 
Five different values of $\sigma$ are considered: $\sigma=1,2,3,15,20$. 

We randomly generate two base functions. We initialize $s_0=2$, $\theta_0=(1,2)^\top$ and $\omega_0=(0.1,0.1)^\top$. At each iteration, we choose $A_t \sim \pi_b$, update $\theta_{t+1}$ and $\omega_{t+1}$ according to Algorithm \ref{al:1}, and compute $\|\nabla J(\theta_t)\|^2$. As for $T$, we consider $T=1000$. 
For the same state and action spaces, we vary the behavior policy and Markov transition kernel, and repeat our experiment for three more times.
We plot the gradient norm as a function of the number of iterations in Fig. \ref{fig: 1}.

\begin{figure}[!ht]
	\centering 
	\subfigure[MDP 1 ]{\includegraphics[width=0.83\linewidth]{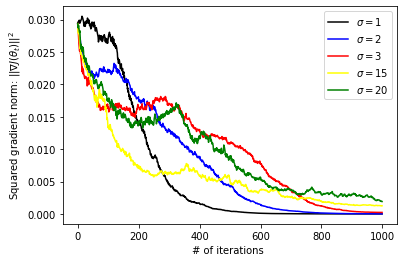}}
	\subfigure[MDP 2 ]{\includegraphics[width=0.83\linewidth]{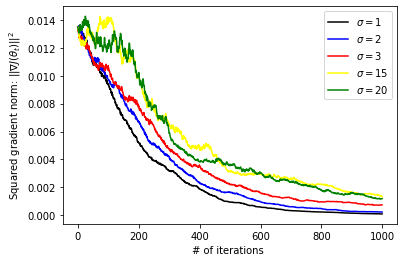}}
	\subfigure[MDP 3 ]{\includegraphics[width=0.83\linewidth]{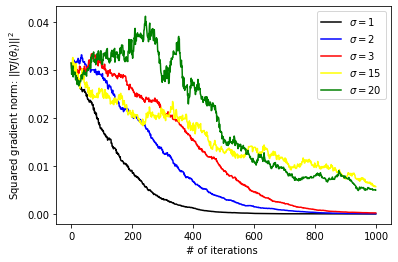}}	\subfigure[MDP 4]{\includegraphics[width=0.83\linewidth]{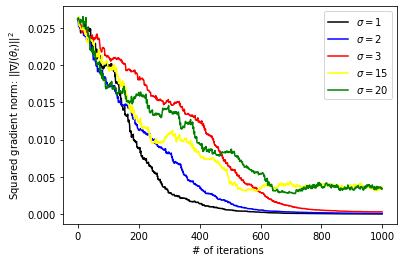}}
	\caption{Comparison among different $\sigma$ for the Greedy-GQ algorithm with softmax operator. } 
	\label{fig: 1}
\end{figure}

It can be seen from Fig. \ref{fig: 1}, as $\sigma$ increases, the convergence of the Greedy-GQ algorithm is getting slower. This observation matches with our theoretical bound that the Greedy-GQ algorithm has a higher sample complexity if the targeted policy is less smoother.

\section{Conclusion}
In this paper, we developed the first finite-sample analysis for the Greedy-GQ algorithm with linear function approximation under Markovian noise. Our analysis is from a novel optimization perspective to solve RL problems. 
We comprehensively characterized the stochastic bias in the gradient estimate and designed a novel technique which recursively applies the obtained bound back into the bias analysis to tighten the convergence rate bound.
We characterized the convergence rate of the Greedy-GQ algorithm, and provided a general guide for choosing stepsizes in practice. The convergence rate obtained by our analysis is $\mathcal O\left(\frac{\log T}{T^{\frac{1}{3}}}\right)$, and is close to the convergence rate $\mathcal O\left(\frac{1}{T^{\frac{1}{2}}}\right)$ for general non-convex optimization problems with unbiased gradient estimate. Such a different is mainly due to the Markovian noise and the tracking error in the two timescale updates.
The techniques developed in this paper may be of independent interest for a wide range of reinforcement learning problems with non-convex objective function and Markovian noise.

In this paper, we provided the finite-sample analysis and the convergence rate for the case with constant stepsizes. The convergence rate for the case with diminishing stepsizes can be derived similarly.
One interesting future direction is to investigate the Greedy-GQ algorithm with the greedy policy. Specifically, $$\pi_\theta(a|s)=1 \text{ if } a=\arg\max_{a'\in\mca} \phi_{s,a}^\top\theta.$$ Due to this max operator, the objective function $J(\theta)$ becomes non-differentiable and non-smooth. To the best of the author's knowledge, there does not exist a general methodology to analyze non-convex non-differentiable optimization problems.  One possible solution is to explore the special geometry of the objective function, i.e., $J(\theta)$ is a piece-wise quadratic function of $\theta$. It is also of further interest to investigate the Greedy-GQ algorithm with general function approximation, e.g., neural network.
\newpage
\bibliographystyle{abbrv}
\bibliography{RL}

\newpage
\onecolumn
\appendix
{\Large\textbf{Supplementary Materials}}

\section{Useful Lemmas for Proving Theorem \ref{thm:main}}\label{app:lemmas}
In this subsection, we prove some useful Lemmas for our finite-sample analysis.

Before we start, we first introduce some nations. In the following proof, $\|a\|$ denotes the $\ell_2$ norm if $a$ is a vector; and $\|A\|$ denotes the operator norm if $A$ is a matrix. Let $\lambda $ be the smallest eigenvalue of the matrix $C$. Then the operator norm of $C^{-1}$ is $\frac{1}{\lambda}$.
We note that the Greedy-GQ algorithm in Algorithm \ref{al:1} was shown to converge asymptotically, and $\theta_t$ and $\omega_t$ were shown to be bounded a.s. (see Proposition 4 in \cite{maei2010toward}). We then define $R$ as the upper bound on both $\theta_t$ and $\omega_t$. Specifically, for any $t$, $\|\theta_t\|\leq R$ and $\|\omega_t\|\leq R$ a.s..

We first prove that if the policy $\pi_\theta$ is smooth in $\theta$, then the object function $J(\theta)$ is also smooth. 
\begin{Lemma}\label{Lemma:1}
The objective function $J(\theta)$ is $K$-smooth for $\theta \in \{\theta: \|\theta\|\leq R\}$, i.e., for any $\|\theta_1\|,\|\theta_2 \| \leq R$,
\begin{flalign}
\|\nabla J(\theta_1)-\nabla J(\theta_2)\| \leq K||\theta_1-\theta_2||,
\end{flalign} 
where 
$
    K=2\gamma\frac{1}{\lambda}\left((k_1|\mca|R+1)(1+\gamma+\gamma Rk_1|\mca|)+|\mca|(r_{\max}+R+\gamma R)( 2k_1+ k_2R) \right).
$
\end{Lemma}
\begin{proof}
Recall the expression of $J\left (\theta\right )$: 
\begin{align}
    J\left (\theta\right )=\mathbb{E}_\mu\left [\delta_{S,A,S'}\left (\theta\right )\phi_{S,A}\right ]^\top C^{-1}\mathbb{E}_\mu\left [\delta_{S,A,S'}\left (\theta\right )\phi_{S,A}\right ],
\end{align}
where $\delta_{S,A,S'}=r_{S,A,S'}+\gamma \sum_{a\in \mathcal{A}} \pi_{\theta}\left (a|S'\right )\theta^\top \phi_{S',a}-\theta^\top \phi_{S,A}$. Then, 
\begin{align}
    \nabla J\left (\theta\right )=2\nabla \left (\mathbb{E}_{\mu}\left [\delta_{S,A,S'}\left (\theta\right )\phi_{S,A}\right ]\right ) C^{-1}\mathbb{E}_{\mu}\left [\delta_{S,A,S'}\left (\theta\right )\phi_{S,A}\right ],
\end{align}
where 
\begin{align}\label{eq:28}
    \nabla \left (\mathbb{E}_{\mu}\left [\delta_{S,A,S'}\left (\theta\right )\phi_{S,A}\right ]\right )&=\mathbb{E}_{\mu}\left [\left (\nabla \gamma\sum_{a\in \mathcal{A}} \pi_{\theta}\left (a|S'\right )\theta^\top \phi_{S',a}\right )\phi_{S,A}^\top \right ]\nn\\
    &=\gamma\mathbb{E}_{\mu}\left [\left (\sum_{a\in\mathcal{A}} \nabla\left (\pi_{\theta}\left (a|S'\right )\right )\theta^\top \phi_{S',a}+\pi_{\theta}\left (a|S'\right )\phi_{S',a}\right )\phi_{S,A}^\top\right ].
\end{align}

It then follows that  
\begin{align}
    &\nabla J\left (\theta_1\right )-\nabla J\left (\theta_2\right )\nn\\
    &=2\nabla \left (\mathbb{E}_{\mu}\left [\delta_{S,A,S'}\left (\theta_1\right )\phi_{S,A}\right ]\right ) C^{-1}\mathbb{E}_{\mu}\left [\delta_{S,A,S'}\left (\theta_1\right )\phi_{S,A}\right ]-2\nabla \left (\mathbb{E}_{\mu}\left [\delta_{S,A,S'}\left (\theta_2\right )\phi_{S,A}\right ]\right ) C^{-1}\mathbb{E}_{\mu}\left [\delta_{S,A,S'}\left (\theta_2\right )\phi_{S,A}\right ]\nn\\
    &=2\nabla \left (\mathbb{E}_{\mu}\left [\delta_{S,A,S'}\left (\theta_1\right )\phi_{S,A}\right ]\right ) C^{-1}\mathbb{E}_{\mu}\left [\delta_{S,A,S'}\left (\theta_1\right )\phi_{S,A}\right ]-2\nabla \left (\mathbb{E}_{\mu}\left [\delta_{S,A,S'}\left (\theta_1\right )\phi_{S,A}\right ]\right ) C^{-1}\mathbb{E}_{\mu}\left [\delta_{S,A,S'}\left (\theta_2\right )\phi_{S,A}\right ]\nn\\
    &+2\nabla \left (\mathbb{E}_{\mu}\left [\delta_{S,A,S'}\left (\theta_1\right )\phi_{S,A}\right ]\right ) C^{-1}\mathbb{E}_{\mu}\left [\delta_{S,A,S'}\left (\theta_2\right )\phi_{S,A}\right ]-2\nabla \left (\mathbb{E}_{\mu}\left [\delta_{S,A,S'}\left (\theta_2\right )\phi_{S,A}\right ]\right ) C^{-1}\mathbb{E}_{\mu}\left [\delta_{S,A,S'}\left (\theta_2\right )\phi_{S,A}\right ].
\end{align}
Since $C^{-1}$ is positive definite, thus to show $\nabla J(\theta)$ is Lipschitz, it suffices to show both $\nabla \left (\mathbb{E}_{\mu}\left [\delta_{S,A,S'}\left (\theta\right )\phi_{S,A}\right ]\right )$ and $\mathbb{E}_{\mu}\left [\delta_{S,A,S'}\left (\theta\right )\phi_{S,A}\right ]$ are Lipschitz in $\theta$ and bounded.

We first show that
\begin{align}\label{eq:30}
    \|\mathbb{E}_{\mu}\left [\delta_{S,A,S'}\left (\theta\right )\phi_{S,A}\right ]\|\leq r_{\max}+(1+\gamma) R,
\end{align}
and
\begin{align}
    \|\nabla \mathbb{E}_{\mu}\left [\delta_{S,A,S'}\left (\theta\right )\phi_{S,A}\right ] \|=\| \mathbb{E}_{\mu}\left [\nabla \delta_{S,A,S'}\left (\theta\right )\phi_{S,A}\right ]\|\leq \gamma(k_1|\mca|R+1).
\end{align}

Following from \eqref{eq:28}, we then have that
\begin{align}\label{eq:32}
    &\nabla \left (\mathbb{E}_{\mu}\left [\delta_{S,A,S'}\left (\theta_1\right )\phi_{S,A}\right ]\right )-\nabla \left (\mathbb{E}_{\mu}\left [\delta_{S,A,S'}\left (\theta_2\right )\phi_{S,A}\right ]\right )\nn\\
    &= \gamma\mathbb{E}_{\mu}\left [\left( \sum_{a\in\mathcal{A}} \nabla\left (\pi_{\theta_1}\left (a|S'\right )\right)\theta_1^\top \phi_{S',a}-\nabla\left (\pi_{\theta_2}\left (a|S'\right )\right)\theta_2^\top \phi_{S',a}+\pi_{\theta_1}\left (a|S'\right )\phi_{S',a}-\pi_{\theta_2}\left (a|S'\right )\phi_{S',a}\right)\phi_{S,A}^\top\right ]\nn\\
    &= \gamma\mathbb{E}_{\mu}\Bigg [\Bigg( \sum_{a\in\mathcal{A}} \nabla\left (\pi_{\theta_1}\left (a|S'\right )\right)\theta_1^\top \phi_{S',a}-\nabla\left (\pi_{\theta_2}\left (a|S'\right )\right)\theta_1^\top \phi_{S',a}+\nabla\left (\pi_{\theta_2}\left (a|S'\right )\right)\theta_1^\top \phi_{S',a}\nn\\
    &\quad-\nabla\left (\pi_{\theta_2}\left (a|S'\right )\right)\theta_2^\top \phi_{S',a}\Bigg)\phi_{S,A}^\top\Bigg ]+\gamma\mathbb{E}_{\mu}\left [\left (\sum_{a\in \mca} \left ( \pi_{\theta_1}\left (a|S'\right )\phi_{S',a}-\pi_{\theta_2}\left (a|S'\right )\phi_{S',a}\right)\right )\phi_{S,A}^\top\right ].
\end{align}
This implies that
\begin{align}\label{eq:33}
   & \|\nabla \left (\mathbb{E}_{\mu}\left [\delta_{S,A,S'}\left (\theta_1\right )\phi_{S,A}\right ]\right )-\nabla \left (\mathbb{E}_{\mu}\left [\delta_{S,A,S'}\left (\theta_2\right )\phi_{S,A}\right ]\right ) \|\nn\\
   &\leq \gamma|\mca|\left(2k_1+ k_2 R \right)\|\theta_1-\theta_2\|,
\end{align}
and thus $\nabla \left (\mathbb{E}_{\mu}\left [\delta_{S,A,S'}\left (\theta\right )\phi_{S,A}\right ]\right )$ is Lipschitz in $\theta$.

Following similar steps, we can also show  that $\mathbb{E}_{\mu}\left [\delta_{S,A,S'}\left (\theta\right )\phi_{S,A}\right ]$ is Lipschitz:
\begin{align}\label{eq:34}
    \|\mathbb{E}_{\mu}\left [\delta_{S,A,S'}\left (\theta_1\right )\phi_{S,A}\right ]-\mathbb{E}_{\mu}\left [\delta_{S,A,S'}\left (\theta_2\right )\phi_{S,A}\right ]\|
    \leq \left (\gamma(|\mca|k_1R+1)+1\right ) \|\theta_1-\theta_2\|.
\end{align}

Now by combining both parts in \eqref{eq:33} and \eqref{eq:34}, we can show that
\begin{align}
    &\|\nabla J\left (\theta_1\right )-\nabla J\left (\theta_2\right )\|\nn\\
    &\leq \|2\nabla \left (\mathbb{E}_{\mu}\left [\delta_{S,A,S'}\left (\theta_1\right )\phi_{S,A}\right ]\right ) C^{-1}\left(\mathbb{E}_{\mu}\left [\delta_{S,A,S'}\left (\theta_1\right )\phi_{S,A}\right ]-\mathbb{E}_{\mu}\left [\delta_{S,A,S'}\left (\theta_2\right )\phi_{S,A}\right ]\right)\|\nn\\
    &\quad+\|2\left(\nabla \left (\mathbb{E}_{\mu}\left [\delta_{S,A,S'}\left (\theta_1\right )\phi_{S,A}\right ]\right )-\nabla \left (\mathbb{E}_{\mu}\left [\delta_{S,A,S'}\left (\theta_2\right )\phi_{S,A}\right ]\right )\right) C^{-1}\mathbb{E}_{\mu}\left [\delta_{S,A,S'}\left (\theta_2\right )\phi_{S,A}\right ]\nn\\
    &\leq 2\gamma(k_1|\mca|R+1)\frac{1}{\lambda}(1+\gamma(1+Rk_1|\mca|)\|\theta_1-\theta_2\|\nn\\
    &\quad+2\frac{1}{\lambda}(r_{\max}+(1+\gamma) R)\gamma|\mca|(2k_1+k_2R)\|\theta_1-\theta_2\|\nn\\
    &= 2\gamma\frac{1}{\lambda}\left((k_1|\mca|R+1)(1+\gamma+\gamma Rk_1|\mca|)+|\mca|(r_{\max}+R+\gamma R)( 2k_1+ k_2R) \right) \|\theta_1-\theta_2\|,
\end{align}
which implies that $\nabla J\left (\theta\right )$ is Lipschitz. This completes the proof.
\end{proof}

Recall that $G_{t+1}(\theta, \omega)=\delta_{t+1}(\theta)\phi_t-\gamma(\omega^T\phi_t)\hat{\phi}_{t+1}(\theta)$, where $\delta_{t+1}(\theta)=r_{t+1}+\gamma\Bar{V}_{t+1}(\theta)-\theta^\top \phi_t$, $\Bar{V}_{t+1}(\theta)=\Bar{V}_{\theta}(S_{t+1})=\sum_{a\in \mathcal{A}}\pi_{\theta}(a|S_{t+1})\theta^\top \phi_{S_{t+1},a}$, and $\hat{\phi}_{t+1}(\theta)=\sum_{a\in\mathcal{A}}\theta^\top\phi_{S_{t+1},a}\nabla \pi_{\theta}(a|S_{t+1})+\pi_{\theta}(a|S_{t+1})\phi_{S_{t+1},a}$. The following Lemma shows that $G_{t+1}(\theta, \omega)$ is Lipschitz in $\omega$, and $G_{t+1}(\theta, \omega^*(\theta))$ is Lipschitz in $\theta$.
\begin{Lemma}\label{Lemma:2}
For any $\theta\in \{\theta:\|\theta\|\leq R\}$, $G_{t+1}(\theta, \omega)$ is Lipschitz in $\omega$, and $G_{t+1}(\theta, \omega^*(\theta))$ is Lipschitz in $\theta$. Specifically, for any $w_1,w_2$, 
\begin{align}
    \|G_{t+1}(\theta,\omega_1)-G_{t+1}(\theta,\omega_2)\|\leq \gamma(|\mca|Rk_1+1)\|\omega_1-\omega_2\|,
\end{align} 
and for any $\theta_1,\theta_2\in \{\theta:\|\theta\|\leq R\}$,
\begin{align}
    &\|G_{t+1}(\theta_1,\omega^*(\theta_1))-G_{t+1}(\theta_2,\omega^*(\theta_2))\| \leq k_3\|\theta_1-\theta_2 \|,
\end{align}
where $k_3=(1+\gamma+\gamma R|\mca|k_1+\gamma\frac{1}{\lambda}|\mca|(2k_1+k_2R)(r_{\max}+\gamma R+R)+\gamma \frac{1}{\lambda}(1+|\mca|Rk_1)(1+\gamma+\gamma R|\mca|k_1)).$
\end{Lemma}
\begin{proof}
Following similar steps as those in \eqref{eq:32} and \eqref{eq:33}, we can show that $\hat{\phi}_{t+1}(\theta)$ is Lipschitz in $\theta$, i.e., for any $\theta_1, \theta_2$ $\in \{\theta:\|\theta\|\leq R\}$,
\begin{align}
    &\|\hat{\phi}_{t+1}(\theta_1)-\hat{\phi}_{t+1}(\theta_2)\|
    \leq | \mathcal{A}| (2k_1  +k_2R)\|\theta_1-\theta_2 \|.
\end{align}

Under Assumption \ref{assump:policy}, it can be easily shown that 
\begin{flalign}\label{eq:39}
\|\hat{\phi}_{t+1}(\theta) \|\leq |\mca|Rk_1+1.
\end{flalign}
It then follows that for any $\omega_1$ and $\omega_2$,
\begin{align}
    &\|G_{t+1}(\theta,\omega_1))-G_{t+1}(\theta,\omega_2))\|\nn\\
    &=\|\gamma(\omega_1-\omega_2)^\top \phi_t) \hat{\phi}_{t+1}(\theta) \|\nn\\
    &\leq \gamma (|\mca|Rk_1+1)\|\omega_1-\omega_2\|.
\end{align}

To show that $G_{t+1}(\theta,\omega^*(\theta))$ is Lipschitz in $\theta$, we have that
\begin{align}
    &\|G_{t+1}(\theta_1,\omega^*(\theta_1))-G_{t+1}(\theta_2,\omega^*(\theta_2)) \|\nn\\
    &\leq |\delta_{t+1}(\theta_1)-\delta_{t+1}(\theta_2) |
    +\gamma\|(\omega^*(\theta_2))^\top\phi_t\hat{\phi}_{t+1}(\theta_2)-(\omega^*(\theta_1))^\top\phi_t\hat{\phi}_{t+1}(\theta_1) \|\nn\\
    &\overset{(a)}{\leq}\gamma\|(\omega^*(\theta_2))^\top \phi_t\hat{\phi}_{t+1}(\theta_2)-(\omega^*(\theta_1))^\top\phi_t\hat{\phi}_{t+1}(\theta_1)-(\omega^*(\theta_1))^\top\phi_t\hat{\phi}_{t+1}(\theta_2)+(\omega^*(\theta_1))^\top\phi_t\hat{\phi}_{t+1}(\theta_2) \|\nn\\
    &\quad+(1+\gamma+\gamma R|\mathcal{A}|k_1)\|\theta_1-\theta_2 \|\nn\\
    &\leq \gamma(1+|\mca|Rk_1)\|\omega^*(\theta_2)-\omega^*(\theta_1) \|+\gamma\|\omega^*(\theta_1) \|\|\hat{\phi}_{t+1}(\theta_1)-\hat{\phi}_{t+1}(\theta_2)\ \|\nn\\
    &\quad+\gamma(1+R|\mathcal{A}|k_1)\|\theta_1-\theta_2 \|+\|\theta_1-\theta_2 \|\nn\\
    &\overset{(b)}{\leq } \left(1+\gamma+\gamma R|\mca|k_1+\gamma\frac{1}{\lambda}|\mca|(2k_1+k_2R)(r_{\max}+\gamma R+R)+\gamma \frac{1}{\lambda}(1+|\mca|Rk_1)(1+\gamma+\gamma R|\mca|k_1)\right)\nn\\
    &\quad\times\|\theta_1-\theta_2 \|\nn\\
    &\triangleq k_3\|\theta_1-\theta_2\|,
\end{align}
where $(a)$ can be shown following steps similar to those  in \eqref{eq:34}, while  $(b)$ can be shown by combining 
\begin{align}
    \|\omega^*(\theta)\|=\|C^{-1}\mathbb{E}[\delta_{t+1}(\theta)\phi_t]\|\leq \frac{1}{\lambda}(r_{\max}+\gamma R+R),
\end{align}
and 
\begin{align}\label{eq:w*lip}
    \|\omega^*(\theta_2)-\omega^*(\theta_1)\|\leq \frac{1}{\lambda}(1+\gamma+\gamma R|\mca|k_1)\|\theta_1-\theta_2\|.
\end{align}

\end{proof}

In the following lemma, we provide a decomposition of the stochastic bias, which is essential to our finite-sample analysis.
\begin{Lemma}
Consider the Greedy-GQ algorithm (see Algorithm \ref{al:1}), when the stepsize $\alpha_t$ is constant, i.e., $\alpha_t=\alpha,\forall t\geq 0$, then
\begin{align}\label{eq:main}
    \sum^T_{t=0} \frac{\alpha_t}{2} \mathbb{E}[\|\nabla J(\theta_t) \|^2] &\leq J(\theta_0)-J(\theta_{T+1})+ \gamma\alpha_t(1+|\mca|Rk_1)\sqrt{\sum^T_{t=0} \mathbb{E}[\|\nabla J(\theta_t)\|^2]}\sqrt{\sum^T_{t=0}\mathbb{E}[\|\omega^*(\theta_t)-\omega_t\|^2]} \nn\\
    &\quad+\sum^T_{t=0}\alpha_t \mathbb{E}[\langle \nabla J(\theta_t),\frac{\nabla J(\theta_t)}{2}+G_{t+1}(\theta_t, \omega^*(\theta_t)) \rangle]+\frac{K}{2}\sum^T_{t=0}\alpha_t^2 \mathbb{E}[\| G_{t+1}(\theta_t,\omega_t)\|^2].
\end{align}
\end{Lemma}

\begin{proof}
From Lemma \ref{Lemma:1}, it follows that $J(\theta)$ is $K$-smooth. Then, by Taylor expansion, for any $\theta_1$ and $\theta_2$, 
\begin{align}
    |J(\theta_1)-J(\theta_2)-\langle \nabla J(\theta_2), \theta_1-\theta_2 \rangle | \leq \frac{K}{2}||\theta_1-\theta_2||^2.
\end{align}
Then, it can be shown that
\begin{align}
    J(\theta_{t+1}) &\leq J(\theta_t) +\langle \nabla J(\theta_t), \theta_{t+1}-\theta_t\rangle  + \frac{K}{2} \alpha_t^2||G_{t+1}(\theta_t,\omega_t)||^2\nn\\
    &=J(\theta_t) +\alpha_t \langle \nabla J(\theta_t),G_{t+1}(\theta_t,\omega_t) \rangle  + \frac{K}{2} \alpha_t^2||G_{t+1}(\theta_t,\omega_t)||^2\nn\\
    &=J(\theta_t)-\alpha_t\langle \nabla J(\theta_t),-G_{t+1}(\theta_t, \omega_t)-\frac{\nabla J(\theta_t)}{2}+G_{t+1}(\theta_t, \omega^*(\theta_t))-G_{t+1}(\theta_t, \omega^*(\theta_t)) \rangle \nn\\
    &\quad-\frac{\alpha_t}{2}||\nabla J(\theta_t)||^2+\frac{K}{2} \alpha_t^2||G_{t+1}(\theta_t,\omega_t)||^2\nn\\
    &=J(\theta_t)-\alpha_t\langle \nabla J(\theta_t),-G_{t+1}(\theta_t, \omega_t)+G_{t+1}(\theta_t, \omega^*(\theta_t)) \rangle\nn\\
    &\quad+\alpha_t \langle \nabla J(\theta_t), \frac{\nabla J(\theta_t)}{2}+G_{t+1}(\theta_t, \omega^*(\theta_t)) \rangle-\frac{\alpha_t}{2}||\nabla J(\theta_t)||^2+\frac{K}{2} \alpha_t^2||G_{t+1}(\theta_t,\omega_t)||^2\nn\\
    &\overset{(a)}{\leq} J(\theta_t) +\alpha_t \gamma\|\nabla J(\theta_t) \|(1+|\mca|Rt_1)\|\omega^*(\theta_t)-\omega_t \|+\alpha_t \langle \nabla J(\theta_t), \frac{\nabla J(\theta_t)}{2}+G_{t+1}(\theta_t, \omega^*(\theta_t)) \rangle \nn\\
    &\quad-\frac{\alpha_t}{2}||\nabla J(\theta_t)||^2+\frac{K}{2} \alpha_t^2||G_{t+1}(\theta_t,\omega_t)||^2,
\end{align}
where $(a)$ follows from the fact that $G_{t+1}(\theta,\omega)$ is Lipschitz in $\omega$ (see Lemma \ref{Lemma:2}).

By taking expectation of both sides, summing up the inequality from $0$ to $T$, and rearranging the terms, we have that
\begin{align}\label{eq:47}
    &\sum^T_{t=0} \frac{\alpha_t}{2} \mathbb{E}[\|\nabla J(\theta_t) \|^2] \nn\\
    &\leq J(\theta_0)-J(\theta_{T+1})+\sum^T_{t=0} \gamma\alpha_t(1+|\mca|Rk_1) \mathbb{E}[\|\nabla J(\theta_t)\|\|\omega^*(\theta_t)-\omega_t\|]\nn\\
    &\quad+\sum^T_{t=0}\alpha_t \mathbb{E}[\langle \nabla J(\theta_t),\frac{\nabla J(\theta_t)}{2}+G_{t+1}(\theta_t, \omega^*(\theta_t)) \rangle]+\frac{K}{2}\sum^T_{t=0}\alpha_t^2 \mathbb{E}[\| G_{t+1}(\theta_t,\omega_t)\|^2].
\end{align}
We then apply Cauchy-Schwarz's inequality,  and we have that
\begin{align}
    &\sum^T_{t=0}\mathbb{E}[\|\nabla J(\theta_t)\|\|\omega^*(\theta_t)-\theta_t\|]\nn\\
    &\leq \sum^T_{t=0}\sqrt{\mathbb{E}[\|\nabla J(\theta_t)\|^2]\mathbb{E}[\|\omega^*(\theta_t)-\theta_t\|^2]}.
\end{align}
We further define two vectors $a_E$ and $a_z$, where 
\begin{flalign}
a_E&\triangleq\left(\sqrt{\mathbb{E}[\|\nabla J(\theta_0)\|^2]},\sqrt{\mathbb{E}[\|\nabla J(\theta_1)\|^2]},...,\sqrt{\mathbb{E}[\|\nabla J(\theta_T)\|^2]}\right)^\top,\\
a_z&\triangleq\left(\sqrt{\mathbb{E}[\|\omega^*(\theta_0)-\theta_0\|^2]},\sqrt{\mathbb{E}[\|\omega^*(\theta_1)-\theta_1\|^2]},...,\sqrt{\mathbb{E}[\|\omega^*(\theta_T)-\theta_T\|^2]}\right)^\top.
\end{flalign}
%
Then, it follows that 
\begin{align}\label{eq:64a}
    &\sum^T_{t=0}\sqrt{\mathbb{E}[\|\nabla J(\theta_t)\|^2]\mathbb{E}[\|\omega^*(\theta_t)-\theta_t\|^2]}\nn\\
    &=\langle a_E,a_z\rangle \nn\\
    &\leq\|a_E\|\|a_z\|\nn\\
    &=\sqrt{\sum^T_{t=0} \mathbb{E}[\|\nabla J(\theta_t)\|^2]}\sqrt{\sum^T_{t=0}\mathbb{E}[\|\omega^*(\theta_t)-\omega_t\|^2]}.
\end{align}

Thus plugging \eqref{eq:64a} in \eqref{eq:47}, and since $\alpha_t=\alpha, \forall t\geq 0$ is constant, we have that 
\begin{align}\label{eq:recursion0}
    &\sum^T_{t=0} \frac{\alpha_t}{2} \mathbb{E}[\|\nabla J(\theta_t) \|^2] \nn\\
    &\leq J(\theta_0)-J(\theta_{T+1})+ \gamma\alpha_t(1+|\mca|Rk_1) \sum^T_{t=0}\mathbb{E}[\|\nabla J(\theta_t)\|\|\omega^*(\theta_t)-\omega_t\|]\nn\\
    &\quad+\sum^T_{t=0}\alpha_t \mathbb{E}[\langle \nabla J(\theta_t),\frac{\nabla J(\theta_t)}{2}+G_{t+1}(\theta_t, \omega^*(\theta_t)) \rangle]+\frac{K}{2}\sum^T_{t=0}\alpha_t^2 \mathbb{E}[\| G_{t+1}(\theta_t,\omega_t)\|^2]\nn\\
    &\leq J(\theta_0)-J(\theta_{T+1})+ \gamma\alpha_t(1+|\mca|Rk_1)\sqrt{\sum^T_{t=0} \mathbb{E}[\|\nabla J(\theta_t)\|^2]}\sqrt{\sum^T_{t=0}\mathbb{E}[\|\omega^*(\theta_t)-\omega_t\|^2]} \nn\\
    &\quad+\sum^T_{t=0}\alpha_t \mathbb{E}[\langle \nabla J(\theta_t),\frac{\nabla J(\theta_t)}{2}+G_{t+1}(\theta_t, \omega^*(\theta_t)) \rangle]+\frac{K}{2}\sum^T_{t=0}\alpha_t^2 \mathbb{E}[\| G_{t+1}(\theta_t,\omega_t)\|^2].
\end{align}
\end{proof}

We next derive the bounds on $ \mathbb{E}[\langle \nabla J(\theta_t),\frac{\nabla J(\theta_t)}{2}+G_{t+1}(\theta_t, \omega^*(\theta_t)) \rangle]$ and $\mathbb{E}[\|\omega^*(\theta_t)-\omega_t\|]$, where we refer to the second term as the "tracking error".

We first define $z_t=\omega_t-\omega^*(\theta_t)$, then the algorithm can be written as:
\begin{align}  
              \theta_{t+1}&=\theta_t+\alpha_t(f_1(\theta_t,O_t)+g_1(\theta_t,z_t,O_t)),   \\  
             z_{t+1}&=z_t+\beta_t(f_2(\theta_t,O_t)+g_2(\theta_t,O_t))+\omega^*(\theta_t)-\omega^*(\theta_{t+1}),  
\end{align}  
where 
\begin{align}
\left\{
\begin{aligned}
&f_1(\theta_t, O_t) \triangleq  \delta_{t+1}(\theta_t)\phi_t-\gamma\phi_t^\top  \omega^*(\theta_t)\hat{\phi}_{t+1}(\theta_t), \\
&g_1(\theta_t, z_t, O_t)  \triangleq  -\gamma\phi_t^\top  z_t\hat{\phi}_{t+1}(\theta_t), \\
&f_2(\theta_t,O_t) \triangleq (\delta_{t+1}(\theta_t)-\phi_t^\top  \omega^*(\theta_t))\phi_t,\\
&g_2(z_t,O_t)  \triangleq  -\phi_t^\top  z_t\phi_t,\\
&O_t\triangleq(s_t, a_t, r_t,s_{t+1}).
\end{aligned}
\right.
\end{align}  
We then develop some upper bounds of functions $f_1,g_1,f_2,g_2$ in the algorithm in the following lemma.
\begin{Lemma}\label{Lemma:3}
For $\|\theta\|\leq R$, $\|z\|\leq 2R$, there exist constants $c_{f_1}$, $c_{g_1}$, $c_{g_2}$ and $c_{f_2}$ such that $\|f_1(\theta,O_t)\|\leq c_{f_1},$ $\|g_1(\theta,z,O_t)\|\leq c_{g_1},$ $|f_2(\theta,O_t)|\leq c_{f_2}$ and $|g_2(\theta,O_t)|\leq c_{g_2}$,
where $c_{f_1}=r_{\max}+(1+\gamma)R+\gamma\frac{1}{\lambda}(r_{\max}+(1 + \gamma)R)(1+R|\mathcal{A}|k_1)$ , $c_{g_1}=2\gamma R(1+R|\mathcal{A}|k_1)$,  $c_{f_2}=r_{\max}+(1+\gamma)R+\frac{1}{\lambda}(r_{\max}+(1 + \gamma)R)$, and $c_{g_2}=2R$.
\end{Lemma}

\begin{proof}
This Lemma can be shown easily using \eqref{eq:30}, \eqref{eq:39} and \eqref{eq:w*lip}.
\end{proof}



We further define $\zeta(\theta, O_t)\triangleq\langle \nabla J(\theta), \frac{\nabla J(\theta)}{2}+G_{t+1}(\theta, \omega^*(\theta)) \rangle$, then we have that $\mathbb{E}_{\mu}[\zeta(\theta, O_t)]=0$ for any fixed $\theta$, where $(S_t,A_t)$ in $O_t$ follow the stationary distribution $\mu$. 
In the following lemma, we provide upper bound on $\mE[\zeta(\theta, O_t)]$.
\begin{Lemma}\label{thm:zeta}
Let $\tau_{\alpha_T}\triangleq \min \left\{k : m\rho^k \leq \alpha_T \right\}$. If $t \leq \tau_{\alpha_T}$, then 
\begin{align}
    \mathbb{E}[\zeta(\theta_t,O_t)] \leq c_{\zeta}(c_{f_1}+c_{g_1})\alpha_0\tau_{\alpha_T},
\end{align}
and if $t > \tau_{\alpha_T}$, then
\begin{align}
    \mathbb{E}[\zeta(\theta_t, O_t)]\leq  k_{\zeta}\alpha_T+c_{\zeta}(c_{f_1}+c_{g_1})\tau_{\alpha_T}\alpha_{t-\tau_{\alpha_T}}.
\end{align}
Where $c_{\zeta}=2\gamma(1+k_1|\mca|R)\frac{1}{\lambda}(r_{\max}+R+\gamma R)(\frac{K}{2}+k_3)+K(r_{\max}+R+\gamma R)(\gamma \frac{1}{\lambda}(1+k_1|\mca|R)+1+\gamma \frac{1}{\lambda}(1+Rk_1|\mca|))$ and $k_{\zeta}=4\gamma(1+k_1R|\mca|)\frac{1}{\lambda}(r_{\max}+R+\gamma R)^2(2\gamma(1+k_1|\mca|R)\frac{1}{\lambda}+1)$.
\end{Lemma}

\begin{proof}
We note that when $\theta$ is fixed, $\mathbb{E}[G_{t+1}(\theta, \omega^*(\theta))]=-\frac{1}{2} \nabla J(\theta)$. We will use this fact and the Markov mixing property to show this Lemma.
Note that for any $\theta_1$ and $\theta_2$, it follows that 
\begin{align}\label{eq:52} 
    &|\zeta(\theta_1,O_t)-\zeta(\theta_2,O_t)|\nn\\
    &=|\langle \nabla J(\theta_1), \frac{\nabla J(\theta_1)}{2}+G_{t+1}(\theta_1, \omega^*(\theta_1)) \rangle-\langle \nabla J(\theta_1), \frac{\nabla J(\theta_2)}{2}+G_{t+1}(\theta_2, \omega^*(\theta_2)) \rangle\nn\\
    &\quad+\langle \nabla J(\theta_1), \frac{\nabla J(\theta_2)}{2}+G_{t+1}(\theta_2, \omega^*(\theta_2)) \rangle-\langle \nabla J(\theta_2), \frac{\nabla J(\theta_2)}{2}+G_{t+1}(\theta_2, \omega^*(\theta_2)) \rangle |.
\end{align}
Since $J(\theta)$ and $\|\nabla J(\theta)\|$ are Lipschitz in $\theta$ by Lemma \ref{Lemma:1}, thus $\zeta(\theta,O_t)$ is also Lipschitz in $\theta$. We then denote its Lipschitz constant by $c_{\zeta}$, i.e.,
\begin{align}
    |\zeta(\theta_1,O_t)-\zeta(\theta_2,O_t)| \leq c_{\zeta} \|\theta_1-\theta_2 \|,
\end{align}
where 
\begin{align}
    c_{\zeta}&=2\gamma(1+k_1|\mca|R)\frac{1}{\lambda}(r_{\max}+R+\gamma R)(\frac{K}{2}+k_3)\nn\\
    &\quad+K(r_{\max}+R+\gamma R)(\gamma \frac{1}{\lambda}(1+k_1|\mca|R)+1+\gamma \frac{1}{\lambda}(1+Rk_1|\mca|)).
\end{align}

Thus from \eqref{eq:52}, it follows that for any $\tau \geq 0$,
\begin{align}\label{eq:55}
    |\zeta(\theta_t,O_t)-\zeta(\theta_{t-\tau},O_t)| \leq  c_{\zeta} \|\theta_t-\theta_{t-\tau} \|\leq  c_{\zeta}(c_{f_1}+c_{g_1})\sum^{t-1}_{k=t-\tau}\alpha_k.
\end{align}

We define an independent random variable $\hat O=(\hat S,\hat A,\hat R,\hat S')$, where $(\hat S,\hat A)\sim\mu$, $\hat S'$ is the subsequent state and $\hat R$ is the reward. Then $\mathbb{E}[\zeta(\theta_{t-\tau},\hat O)]=0$ by the fact that $\mathbb{E}_{\mu}[{G_{t+1}(\theta,\omega^*(\theta))}]=-\frac{1}{2}\nabla J(\theta)$.
Thus, 
\begin{align}
    \mathbb{E}[\zeta(\theta_{t-\tau},O_t)] \leq |\mathbb{E}[\zeta(\theta_{t-\tau},O_t)]-\mathbb{E}[\zeta(\theta_{t-\tau},O')]|\leq k_{\zeta}m\rho^{\tau},
\end{align}
which follows from the Markov Mixing property in Assumption \ref{ass:1}, where $k_{\zeta}=4\gamma(1+k_1R|\mca|)\frac{1}{\lambda}(r_{\max}+R+\gamma R)^2(2\gamma(1+k_1|\mca|R)\frac{1}{\lambda}+1)$.

If $t \leq \tau_{\alpha_T}$, then we choose $\tau=t$ in \eqref{eq:55}. Then we have that 
\begin{align}
    \mathbb{E}[\zeta(\theta_t,O_t)] \leq \mathbb{E}[\zeta(\theta_0,O_t)]+c_{\zeta}(c_{f_1}+c_{g_1})\sum^{t-1}_{k=0}\alpha_k\leq c_{\zeta}(c_{f_1}+c_{g_1})t\alpha_0\overset{(a)}{\leq} c_{\zeta}(c_{f_1}+c_{g_1})\alpha_0\tau_{\alpha_T},
\end{align}
where $(a)$ is due to the fact that $\alpha_t$ is non-increasing. 
If $t > \tau_{\alpha_T}$, we choose $\tau=\tau_{\alpha_T}$, and then
\begin{align}
    &\mathbb{E}[\zeta(\theta_t, O_t)]\leq \mathbb{E}[\zeta(\theta_{t-\tau_{\alpha_T}},O_t)]+c_{\zeta}(c_{f_1}+c_{g_1})\sum^{t-1}_{k=t-\tau_{\alpha_T}}\alpha_k\nn\\
    &\leq k_{\zeta}m\rho^{\tau_{\alpha_T}}+c_{\zeta}(c_{f_1}+c_{g_1})\tau_{\alpha_T}\alpha_{t-\tau_{\alpha_T}}\leq k_{\zeta}\alpha_T+c_{\zeta}(c_{f_1}+c_{g_1})\tau_{\alpha_T}\alpha_{t-\tau_{\alpha_T}}.
\end{align}
\end{proof}

We next bound the tracking error $\mathbb{E}[\|z_t\|]$.
Define $\zeta_{f_2}(\theta,z,O_t)\triangleq\langle z, f_2(\theta,O_t) \rangle $, and $\zeta_{g_2}(z,O_t)\triangleq\langle z,g_2(z,O_t)-\Bar{g_2}(z)\rangle$, where $\Bar{g_2}(z)\triangleq\mathbb{E}[g_2(z,O_t)]=\mathbb{E}[-\phi_t^\top z \phi_t]$.

\begin{Lemma}\label{Lemma:5}
Consider any $\theta,\theta_1,\theta_2 \in\{\theta:\|\theta\|\leq R\}$ and any $z,z_1,z_2\in\{z:\|z\|\leq 2R\}$. Then
1) $|\zeta_{f_2}(\theta,z,O_t) | \leq 2Rc_{f_2}$;
2) $|\zeta_{f_2}(\theta_1,z_1,O_t)-\zeta_{f_2}(\theta_2,z_2,O_t) | \leq k_{f_2}\|\theta_1-\theta_2 \|+k'_{f_2}\|z_1-z_2\|$, where $k_{f_2}=2R(1+\gamma+\gamma Rk_1|\mca|)(1+\frac{1}{\lambda})$ and $k'_{f_2}=c_{f_2}$;
3) $|\zeta_{g_2}(z,O_t) | \leq 8R^2 $;
and
4) $|\zeta_{g_2}(z_1,O_t)-\zeta_{g_2}(z_2,O_t) | \leq 8R\|z_1-z_2\|$.
\end{Lemma}
\begin{proof}
To prove 1), it can be shown that $|\zeta_{f_2}(\theta,z,O_t)|=|\langle z, f_2(\theta,O_t) \rangle|\leq 2Rc_{f_2}$.

For 2), it can be shown that
\begin{align}
    &|\zeta_{f_2}(\theta_1,z_1,O_t)-\zeta_{f_2}(\theta_2,z_2,O_t)|\nn\\
    &=|\langle z_1, f_2(\theta_1,O_t) \rangle-\langle z_2, f_2(\theta_2,O_t) \rangle|\nn\\
    &\leq |\langle z_1, f_2(\theta_1,O_t) \rangle-\langle z_1, f_2(\theta_2,O_t)|+|\langle z_1, f_2(\theta_2,O_t)-\langle z_2, f_2(\theta_2,O_t) \rangle|\nn\\
    &\leq 2R \| f_2(\theta_1,O_t)-f_2(\theta_2,O_t)\|+\|f_2(\theta_2,O_t) \| \|z_1-z_2 \|\nn\\
    &\leq 2R(|\delta_{t+1}(\theta_1)-\delta_{t+1}(\theta_2)|+\|\omega^*(\theta_1)-\omega^*(\theta_2) \|)+c_{f_2}\|z_1-z_2 \|\nn\\
    &\overset{(a)}{\leq} k_{f_2}\|\theta_1-\theta_2\|+k'_{f_2}\|z_1-z_2\|,
\end{align}
where $(a)$ is from both $\delta(\theta)$ and $\omega^*(\theta_t)(\theta)$ are Lipschitz, $k_{f_2}=2R(1+\gamma+\gamma Rk_1|\mca|)(1+\frac{1}{\lambda})$, and $k'_{f_2}=c_{f_2}$.

For 3), we have that $\zeta_{g_2}(z,O_t)=\langle z, -\phi_t^\top z\phi_t+\mathbb{E}[\phi_t^\top z\phi_t]\rangle \leq 8R^2$.

To prove 4), we have that 
\begin{align}
    &|\zeta_{g_2}(z_1,O_t)-\zeta_{g_2}(z_2,O_t)|\nn\\
    &=|\langle z_1, -\phi_t^\top z_1\phi_t+\mathbb{E}[\phi_t^\top z_1\phi_t]\rangle-\langle z_1, -\phi_t^\top z_2\phi_t+\mathbb{E}[\phi_t^\top z_2\phi_t]\rangle+\langle z_1, -\phi_t^\top z_2\phi_t\nn\\
    &\quad+\mathbb{E}[\phi_t^\top z_2\phi_t]\rangle-\langle z_2, -\phi_t^\top z_2\phi_t+\mathbb{E}[\phi_t^\top z_2\phi_t]\rangle|\nn\\
    &\leq 8R\|z_1-z_2\|.
\end{align}
\end{proof}

In the following lemma,  we derive bounds on $\mE[\zeta_{f_2}(\theta_1,z_t,O_t)]$ and $\mE[\zeta_{g_2}(z_t,O_t)]$.
\begin{Lemma}\label{lemma:6}
Define $\tau_{\beta_T}=\min \left\{ k: m\rho^k \leq \beta_T \right\}$.
If $t\leq \tau_{\beta_T}$, then
\begin{align}
    \mathbb{E}[\zeta_{f_2}(\theta_t,z_t,O_t)]\leq 4Rc_{f_2}\beta_T +a_{f_2}\tau_{\beta_T},
\end{align} where $a_{f_2}=(k'_{f_2}(c_{f_2}+c_{g_2})\beta_0+ (k_{f_2}(c_{f_1}+c_{g_1})+k'_{f_2}\frac{1}{\lambda}(1+\gamma+\gamma R|\mca|k_1)(c_{f_1}+c_{g_1}))\alpha_0)$;
and if $t> \tau_{\beta_T}$, then
\begin{align}
    \mathbb{E}[\zeta_{f_2}(\theta_t,z_t,O_t)]\leq 4Rc_{f_2}\beta_T+b_{f_2}\tau_{\beta_T}\beta_{t-\tau_{\beta_T}},
\end{align}
where $b_{f_2}=( k'_{f_2}(c_{f_2}+c_{g_2})+ (k_{f_2}(c_{f_1}+c_{g_1})+k'_{f_2}\frac{1}{\lambda}(1+\gamma+\gamma R|\mca|k_1)(c_{f_1}+c_{g_1})))$.
\end{Lemma}
\begin{proof}
We first note that 
\begin{align}
    &\|z_{t+1}-z_t\|\nn\\
    &=\|\beta_t(f_2(\theta_t,O_t)+g_2(z_t,O_t))+\omega^*(\theta_t)-\omega^*(\theta_{t+1}) \|\nn\\
    &\leq (c_{f_2}+c_{g_2})\beta_t+\frac{1}{\lambda}(1+\gamma+\gamma R|\mca|k_1)(c_{f_1}+c_{g_1})\alpha_t,
\end{align}
where the last step is due to \eqref{eq:w*lip}.
Furthermore, due to part 2) in Lemma \ref{Lemma:5}, $\zeta_{f_2}$ is Lipschitz in both $\theta$ and $z$, then we have that for any $\tau\geq 0$
\begin{align}\label{eq:64}
    &|\zeta_{f_2}(\theta_t,z_t,O_t)-\zeta_{f_2}(\theta_{t-\tau},z_{t-\tau},O_t)|\nn\\
    &\overset{(a)}{\leq} k_{f_2}(c_{f_1}+c_{g_1})\sum^{t-1}_{i=t-\tau}\alpha_i+k'_{f_2}(c_{f_2}+c_{g_2})\sum^{t-1}_{i=t-\tau}\beta_i+\sum^{t-1}_{i=t-\tau}k'_{f_2}\frac{1}{\lambda}(1+\gamma+\gamma R|\mca|k_1)(c_{f_1}+c_{g_1})\alpha_i\nn\\
    &=k'_{f_2}(c_{f_2}+c_{g_2})\sum^{t-1}_{i=t-\tau}\beta_i+(k_{f_2}(c_{f_1}+c_{g_1})+k'_{f_2}\frac{1}{\lambda}(1+\gamma+\gamma R|\mca|k_1)(c_{f_1}+c_{g_1}))\sum^{t-1}_{i=t-\tau}\alpha_i,
\end{align}
where in $(a)$, we apply \eqref{eq:w*lip} and Lemma \ref{Lemma:3} to obtain the third term. 

Define an independent random variable $\hat O=(\hat S,\hat A,\hat R,\hat S')$, where $(\hat S,\hat A)\sim \mu $, $\hat S'\sim\mathsf P(\cdot|\hat S,\hat A)$ is the subsequent state, and $\hat R$ is the reward. Then it can be shown that
\begin{align}
    &\mathbb{E}[\zeta_{f_2}(\theta_{t-\tau},z_{t-\tau},O_t)] \nn\\
    &\overset{(a)}{\leq} |\mathbb{E}[\zeta_{f_2}(\theta_{t-\tau},z_{t-\tau},O_t)]-\mathbb{E}[\zeta_{f_2}(\theta_{t-\tau},z_{t-\tau},\hat O)]|\nn\\
    &\leq 4Rc_{f_2}m\rho^{\tau},
\end{align}
where (a) is due to the fact that $\mathbb{E}[\zeta_{f_2}(\theta_{t-\tau},z_{t-\tau},\hat O)]=0$, and the last inequality follows from Assumption \ref{ass:1}.

If $t\leq \tau_{\beta_T}$, we choose $\tau=t$ in \eqref{eq:64}. Then it can be shown that 
\begin{align}
    &\mathbb{E}[\zeta_{f_2}(\theta_t,z_t,O_t)]\nn\\
    &\leq  \mathbb{E}[\zeta_{f_2}(\theta_0,z_0,O_t)]+k'_{f_2}(c_{f_2}+c_{g_2})\sum^{t-1}_{i=0}\beta_i+(k_{f_2}(c_{f_1}+c_{g_1})\nn\\
    &\quad+k'_{f_2}\frac{1}{\lambda}(1+\gamma+\gamma R|\mca|k_1)(c_{f_1}+c_{g_1}))\sum^{t-1}_{i=0}\alpha_i\nn\\
    &\leq 4Rc_{f_2}m\rho^t+ k'_{f_2}(c_{f_2}+c_{g_2})t\beta_0+(k_{f_2}(c_{f_1}+c_{g_1})+k'_{f_2}\frac{1}{\lambda}(1+\gamma+\gamma R|\mca|k_1)(c_{f_1}+c_{g_1}))t\alpha_0\nn\\
    &\leq 4Rc_{f_2}\beta_T + (k'_{f_2}(c_{f_2}+c_{g_2})\beta_0+ (k_{f_2}(c_{f_1}+c_{g_1})+k'_{f_2}\frac{1}{\lambda}(1+\gamma+\gamma R|\mca|k_1)(c_{f_1}+c_{g_1}))\alpha_0)\tau_{\beta_T}.
\end{align}
If $t> \tau_{\beta_T}$, we choose $\tau=\tau_{\beta_T}$ in \eqref{eq:64}. Then, it can be shown that 
\begin{align}
    &\mathbb{E}[\zeta_{f_2}(\theta_t,z_t,O_t)]\nn\\
    &\leq  \mathbb{E}[\zeta_{f_2}(\theta_{t-\tau_{\beta_T}},z_{t-\tau_{\beta_T}},O_t)]\nn\\
    &\quad+k'_{f_2}(c_{f_2}+c_{g_2})\sum^{t-1}_{i=t-\tau_{\beta_T}}\beta_i+(k_{f_2}(c_{f_1}+c_{g_1})+k'_{f_2}\frac{1}{\lambda}(1+\gamma+\gamma R|\mca|k_1)(c_{f_1}+c_{g_1}))\sum^{t-1}_{i=t-\tau_{\beta_T}}\alpha_i\nn\\
    &\leq 4Rc_{f_2}m\rho^{\tau_{\beta_T}}+k'_{f_2}(c_{f_2}+c_{g_2})\tau_{\beta_T}\beta_{t-\tau_{\beta_T}}+(k_{f_2}(c_{f_1}+c_{g_1})+k'_{f_2}\frac{1}{\lambda}(1+\gamma+\gamma R|\mca|k_1)(c_{f_1}+c_{g_1}))\tau_{\beta_T}\alpha_{t-\tau_{\beta_T}}\nn\\
    &\leq 4Rc_{f_2}\beta_T+( k'_{f_2}(c_{f_2}+c_{g_2})+ (k_{f_2}(c_{f_1}+c_{g_1})+k'_{f_2}\frac{1}{\lambda}(1+\gamma+\gamma R|\mca|k_1)(c_{f_1}+c_{g_1})))\tau_{\beta_T}\beta_{t-\tau_{\beta_T}},
\end{align}
where in the last step we upper bound $\alpha_t$ using $\beta_t$. Note that this will not change the order of the bound.
\end{proof}

Similarly, in the following lemma, we derive a bound on $\mathbb{E}[\zeta_{g_2}(z_t,O_t)]$.
\begin{Lemma}\label{lemma:8}
If $t\leq \tau_{\beta_T}$, then
\begin{flalign}
\mathbb{E}[\zeta_{g_2}(z_t,O_t)] \leq a_{g_2}\tau_{\beta_T};
\end{flalign}
and if $t> \tau_{\beta_T}$, then
\begin{flalign}
\mathbb{E}[\zeta_{g_2}(z_t,O_t)] \leq b_{g_2}\beta_T+b'_{g_2}\tau_{\beta_T}\beta_{t-\tau_{\beta_T}},
\end{flalign}
where $a_{g_2}=8R(c_{f_2}+c_{g_2})\beta_0+\frac{1}{\lambda}(1+\gamma+\gamma R|\mca|k_1)(c_{f_1}+c_{g_1})\alpha_0)$, 
$b_{g_2}=16R^2$, and $b_{g_2}'=8R(c_{f_2}+c_{g_2})\beta_0+\frac{1}{\lambda}(1+\gamma+\gamma R|\mca|k_1)(c_{f_1}+c_{g_1})\alpha_0$.
\end{Lemma}
\begin{proof}
The proof is similar to the one for Lemma \ref{lemma:6}.
\end{proof}

We then bound the tracking error as follows:
\begin{align}\label{z_1}
    &||z_{t+1}||^2\nonumber\\
    &=||z_t+\beta_t(f_2(\theta_t,O_t)+g_2(z_t,O_t))+\omega^*(\theta_t)-\omega^*(\theta_{t+1})||^2\nonumber\\
    &=||z_t||^2+2\beta_t \langle z_t, f_2(\theta_t,O_t)\rangle +2\beta_t\langle z_t,g_2(z_t,O_t)\rangle +2\langle z_t, \omega^*(\theta_t)-\omega^*(\theta_{t+1})\rangle\nn\\ &\quad+||\beta_tf_2(\theta_t,O_t)+\beta_tg_2(z_t,O_t)+\omega^*(\theta_t)-\omega^*(\theta_{t+1})||^2\nonumber\\
    &\leq||z_t||^2+2\beta_t \langle z_t, f_2(\theta_t,O_t)\rangle+2\beta_t\langle z_t,g_2(z_t,O_t)\rangle +2\langle z_t, \omega^*(\theta_t)-\omega^*(\theta_{t+1})\rangle\nn\\ &\quad+3\beta_t^2||f_2(\theta_t,O_t)||^2+3\beta_t^2||g_2(z_t,O_t)||^2+3||\omega^*(\theta_t)-\omega^*(\theta_{t+1})||^2\nonumber\\
    &\overset{(a)}{\leq}||z_t||^2+2\beta_t\langle z_t, f_2(\theta_t,O_t)\rangle +2\beta_t\langle z_t,\Bar{g_2}(z_t)\rangle +2\langle z_t, \omega^*(\theta_t)-\omega^*(\theta_{t+1})\rangle +2\beta_t\langle z_t,g_2(z_t,O_t)-\Bar{g_2}(z_t)\rangle \nn\\
    &\quad+3\beta_t^2c_{f_2}^2+3\beta_t^2c_{g_2}^2+6\frac{1}{\lambda^2}(1+\gamma+\gamma R|\mca|k_1)^2\alpha_t^2 (c_{f_1}^2+c_{g_1}^2),
\end{align}
where $(a)$ follows from Lemma \ref{Lemma:3} and \eqref{eq:w*lip}.

Note that $\langle z_t,\Bar{g_2}(z_t)\rangle=-z_t^\top C z_t$, and $C$ is a positive definite matrix. Recall the minimal eigenvalue of $C$ is denoted by $\lambda$, then \eqref{z_1} can be further bounded as follows:
\begin{align}\label{eq:69}
    ||z_{t+1}||^2
    &\leq (1-2\beta_t\lambda)\|z_t\|^2+2\beta_t\zeta_{f_2}+2\beta_t\zeta_{g_2}+2\langle z_t,\omega^*(\theta_t)-\omega^*(\theta_{t+1})\rangle+3\beta_t^2c_{f_2}^2\nonumber\\
    &\quad+3\beta_t^2c_{g_2}^2+6\frac{1}{\lambda^2}(1+\gamma+\gamma R|\mca|k_1)^2\alpha_t^2 (c_{f_1}^2+c_{g_1}^2).
\end{align}
Taking expectation on both sides of the \eqref{eq:69}, and applying it recursively, we obtain that
\begin{align}
    \mathbb{E}[||z_{t+1}||^2] 
    \leq &\prod_{i=0}^{t}(1-2\beta_i\lambda) ||z_0||^2 \nonumber\\
    &+2\sum_{i=0}^{t} \prod_{k=i+1}^{t}(1-2\beta_k\lambda) \beta_i \mathbb{E}[\zeta_{f_2}(z_i,\theta_i,O_i)] \nonumber\\
    &+2\sum_{i=0}^{t} \prod_{k=i+1}^{t}(1-2\beta_k\lambda) \beta_i \mathbb{E}[\zeta_{g_2}(z_i,O_i)]\nonumber\\
    &+2\sum_{i=0}^{t} \prod_{k=i+1}^{t}(1-2\beta_k\lambda) \mathbb{E}\langle z_i, \omega^*(\theta_i)-\omega^*(\theta_{i+1})\rangle 
    +3(c_{f_2}^2+c_{g_2}^2)\sum_{i=0}^{t} \prod_{k=i+1}^{t}(1-2\beta_k\lambda) \beta_i^2\nonumber\\
    &+6\frac{1}{\lambda^2}(1+\gamma+\gamma R|\mca|k_1)^2(c_{f_1}^2+c_{g_1}^2)\sum_{i=0}^{t} \prod_{k=i+1}^{t}(1-2\beta_k\lambda) \alpha_i^2.
\end{align}

Also note that $1-2\beta_i\lambda \leq e^{-2\beta_i\lambda}$, which further implies that
\begin{align}\label{eq:tracking}
    \mathbb{E}[||z_{t+1}||^2
    &\leq A_t ||z_0||^2+2\sum_{i=0}^t B_{it} +2\sum_{i=0}^t C_{it}+2\sum_{i=0}^t D_{it}\nn\\
    &\quad+3(c_{f_2}^2+c_{g_2}^2+2\frac{1}{\lambda^2}(1+\gamma+\gamma R|\mca|k_1)^2(c_{f_1}^2+c_{g_1}^2)) \sum_{i=0}^t E_{it},
\end{align}
where 
\begin{align}\label{ABCDE}
A_t&=e^{-2\lambda \sum_{i=0}^t  \beta_i}, \nn\\
B_{it}&=e^{-2\lambda\sum_{k=i+1}^t \beta_k} \beta_i\mathbb{E}[\zeta_{f_2}(z_i,\theta_i,O_i)], \nn\\
C_{it}&=e^{-2\lambda\sum_{k=i+1}^t  \beta_k} \beta_i\mathbb{E}[\zeta_{g_2}(z_i,O_i)],\nn\\
D_{it}&=e^{-2\lambda\sum_{k=i+1}^t  \beta_k} \mathbb{E}[\langle z_t,\omega^*(\theta_i)-\omega^*(\theta_{i+1})\rangle],\nn\\
E_{it}&=e^{-2\lambda\sum_{k=i+1}^t  \beta_k} \beta_i^2.
\end{align}


Consider the second term in \eqref{eq:tracking}. Using Lemma \ref{lemma:6}, it can be further bounded as follows: 
\begin{align}\label{eq:B}
    \sum^t_{i=0} B_{it}
    &=\sum^t_{i=0}e^{-2\lambda\sum_{k=i+1}^t  \beta_k} \beta_i\mathbb{E}[\zeta_{f_2}(z_i,\theta_i,O_i)]\nonumber\\
    &\leq \sum_{i=0}^{\tau_{\beta_T}}(a_{f_2}\tau_{\beta_T}+4Rc_{f_2}\beta_T)e^{-2\lambda\sum_{k=i+1}^t\beta_k}\beta_i+4Rc_{f_2}\beta_T\sum_{i={\tau_{\beta_T}+1}}^te^{-2\lambda\sum_{k=i+1}^t\beta_k}\beta_i\nn\\
    &\quad+b_{f_2}\tau_{\beta_T}\sum_{i={\tau_{\beta_T}+1}}^te^{-2\lambda\sum_{k=i+1}^t\beta_k}\beta_{i-\tau_{\beta_T}}\beta_i.
\end{align}
Further analysis of the bound will be made when we specify the stepsizes $\alpha_t,\beta_t$, which will be provided later.

Similarly, using Lemma \ref{lemma:8}, we can bound the third term in \eqref{eq:tracking} as follows:
\begin{align}\label{eq:C}
    \sum^t_{i=0}C_{it}&=\sum^t_{i=0} e^{-2\lambda\sum_{k=i+1}^t  \beta_k} \beta_i\mathbb{E}[\zeta_{g_2}(z_i,O_i)]\nn\\
    &\leq\tau_{\beta_T}a_{g_2}\sum_{i=0}^{\tau_{\beta_T}}e^{-2\lambda\sum_{k=i+1}^t\beta_k}\beta_i+b_{g_2}\beta_T\sum_{i={\tau_{\beta_T}+1}}^te^{-2\lambda\sum_{k=i+1}^t\beta_k}\beta_i\nn\\
    &\quad+b'_{g_2}\tau_{\beta_T}\sum_{i={\tau_{\beta_T}+1}}^te^{-2\lambda\sum_{k=i+1}^t\beta_k}\beta_{i-\tau_{\beta_T}}\beta_i.
\end{align}

The last step in bounding the tracking error is to bound $\mathbb{E}[\langle z_i,\omega^*(\theta_i)-\omega^*(\theta_{i+1})\rangle]$, which is shown in the following lemma.
\begin{Lemma}\label{thm:D}
\begin{align}
    &\sum_{i=0}^{t}e^{-2\lambda\sum_{k=i+1}^t  \beta_k} \mathbb{E}[\langle z_i,\omega^*(\theta_i)-\omega^*(\theta_{i+1})\rangle]\nn\\
    &\leq 2\frac{1}{\lambda}(1+\gamma+\gamma R|\mca|k_1)R(c_{f_1}+c_{g_1}) \sum_{i=0}^{t}e^{-2\lambda\sum_{k=i+1}^t   \beta_k} \alpha_i.
\end{align}
\end{Lemma}

\begin{proof}
From \eqref{eq:w*lip}, we first have that  
\begin{flalign}
||\omega^*(\theta_i)-\omega^*(\theta_{i+1})|| \leq \frac{1}{\lambda}(1+\gamma+\gamma R|\mca|k_1)||\theta_i-\theta_{i+1}||.
\end{flalign}
Then it follows that 
\begin{align}\label{eq:81a}
    &\sum_{i=0}^{t}e^{-2\lambda\sum_{k=i+1}^t  \beta_k} \mathbb{E}[\langle z_i,\omega^*(\theta_i)-\omega^*(\theta_{i+1})\rangle]\nn\\
    &\leq \sum_{i=0}^{t}e^{-2\lambda\sum_{k=i+1}^t   \beta_k} \mathbb{E}[\frac{1}{\lambda}(1+\gamma+\gamma R|\mca|k_1) \|z_i\| \|\theta_i-\theta_{i+1}\|]\nn\\
    &\leq 2\frac{1}{\lambda}(1+\gamma+\gamma R|\mca|k_1)R(c_{f_1}+c_{g_1}) \sum_{i=0}^{t}e^{-2\lambda\sum_{k=i+1}^t   \beta_k} \alpha_i.
\end{align}
\end{proof}

\section{Proof of Theorem \ref{thm:main}}
In this section, we will use the lemmas in Appendix \ref{app:lemmas} to prove Theorem \ref{thm:main}.

In Appendix \ref{app:lemmas}, we have developed bounds on both the tracking error and $\mathbb{E}[\zeta(\theta_t,O_t)]$. We then plug them both into \eqref{eq:47},
\begin{align}\label{eq:80}
    &\frac{\sum^T_{t=0}\alpha_t\mathbb{E}[\|\nabla J(\theta_t)\|^2]}{2\sum^T_{t=0} \alpha_t}\nn\\
    &\leq \frac{1}{\sum^T_{t=0} \alpha_t} \bigg( J(\theta_0)-J^*+\gamma\alpha_t(1+|\mca|Rk_1)\sqrt{\sum^T_{t=0}\mathbb{E}[\|\nabla J(\theta_t) \|^2]}\sqrt{\sum^T_{t=0}\mathbb{E}[\|z_t\|^2]}\nn\\
    &\quad+\sum^T_{t=0}\alpha_t \mathbb{E}[\zeta(\theta_t,O_t)]+\sum^T_{t=0}\alpha_t^2 (c_{f_1}+c_{g_1}) \bigg),
\end{align}
where $J^*$ denotes $\min_\theta J(\theta)$, and is positive and finite. 



By Lemma \ref{thm:zeta}, for large $T$, we have that
\begin{flalign}\label{eq:81}
&\sum_{t=0}^T \alpha_t\mathbb{E}[\zeta(\theta_t,O_t)]\nn\\
&\leq \sum^{\tau_{\alpha_T}}_{t=0}c_{\zeta}(c_{f_1}+c_{g_1})\alpha_0\alpha_t\tau_{\alpha_T}+\sum^T_{t=\tau_{\alpha_T}+1} k_{\zeta}\alpha_T\alpha_t+c_{\zeta}(c_{f_1}+c_{g_1})\tau_{\alpha_T}\alpha_{t-\tau_{\alpha_T}}\alpha_t.
\end{flalign}
Here, $\tau_{\alpha_T}=\mathcal O(|\log \alpha_T|)$ by its definition. Therefore, for non-increasing sequence $\{\alpha_t\}_{t=0}^\infty$, \eqref{eq:81} can be further upper bounded as follows:
\begin{flalign}\label{eq:79}
&\sum_{t=0}^T \alpha_t\mathbb{E}[\zeta(\theta_t,O_t)]=\mathcal O\left(|\log \alpha_T|^2\alpha_0^2 +  \sum_{t=0}^T\left(  \alpha_t\alpha_T+|\log{\alpha_T}|\alpha_t^2\right)\right).
\end{flalign}

We note that we can also specify the constants for \eqref{eq:79}, which, however, will be cumbersome. How those constants affect the finite-sample bound can be easily inferred from \eqref{eq:81}, and thus is not explicitly analyzed in the following steps. Also, at the beginning we bound $\sqrt{\frac{\sum^T_{t=0}\mathbb{E}[\|\nabla J(\theta_t)\|^2]}{T}}$ by some constant that does not scale with $T$: $\gamma \|C^{-1}\|(k_1+|\mca|R+1)(r_{\max}+R+\gamma R)$.

Hence, we have that
 \begin{align}\label{eq:main0}
     &\frac{\sum^T_{t=0}\alpha_t\mathbb{E}[\|\nabla J(\theta_t)\|^2]}{\sum^T_{t=0} \alpha_t}\nn\\
     &=\mathcal O\Bigg(  \frac{1}{\sum^T_{t=0} \alpha_t} \Bigg(J(\theta_0)-J^*+\sum^T_{t=0}\alpha_t^2 +\alpha_t\sqrt{T}\sqrt{\sum^T_{t=0}\mathbb{E}[\|z_t\|^2]}+\alpha_0^2|\log (\alpha_T)|^2 +\sum^{T}_{t=0}\alpha_t\alpha_T\nn\\
     &\quad+\sum^T_{t=0}|\log(\alpha_T)|\alpha_t^2 \Bigg)\Bigg).
 \end{align}

%

In the following, we focus on the case with constant stepsizes. For other possible choices of stepsizes, the convergence rate can also be derived using \eqref{eq:main0}.
Let $\alpha_t=\frac{1}{T^a}=\alpha$ and $\beta_t=\frac{1}{T^b}=\beta$.
In this case, \eqref{eq:main0} can be written as follows:
\begin{align}\label{eq:main1}
    \frac{\sum^T_{t=0}\alpha\mathbb{E}[\|\nabla J(\theta_t)\|^2]}{\sum^T_{t=0} \alpha}&=\mathcal O \left( \frac{1}{T}\left(\sqrt{T}\sqrt{\sum^T_{t=0}\mathbb{E}[\|z_t\|^2]}+\alpha\log(\alpha)^2+T\alpha+T\alpha|\log(\alpha)| \right)+\frac{J(\theta_0)-J^*}{T\alpha} \right)\nn\\
    &=\mathcal{O}\left(\sqrt{\frac{\sum^T_{t=0}\mathbb{E}[\|z_t\|^2]}{T}}\right)+\mathcal{O}\left(\frac{\log T^2}{T^{1+a}}+\frac{1}{T^a}+\frac{\log T}{T^a}+\frac{1}{T^{1-a}}\right).
\end{align}

We then consider the tracking error $\mathbb{E}[\|z_{t} \|^2]$. Applying \eqref{eq:tracking}, \eqref{eq:B}, \eqref{eq:C} and \eqref{eq:81a}, we obtain that for $t>\tau_{\beta_T}$,
\begin{align}\label{eq:order}
    &\mathbb{E}[\|z_{t} \|^2]\nn\\
    &\leq  \|z_0\|^2 e^{-2\lambda t\beta}\nn\\
    &\quad+2(4Rc_{f_2}\beta+(a_{f_2}+a_{g_2})\tau_{\beta_T})\beta\sum_{i=0}^{\tau_{\beta_T}}e^{-2\lambda(t-i)\beta}+(8Rc_{f_2}+2b_{g_2})\beta^2\sum_{i={\tau_{\beta_T}+1}}^t e^{-2\lambda(t-i)\beta}\nn\\
    &\quad+(2b_{f_2}+2b'_{g_2})\tau_{\beta_T}\beta^2\sum_{i={\tau_{\beta_T}+1}}^t e^{-2\lambda(t-i)\beta}+\frac{4}{\lambda}(1+\gamma+\gamma R|\mca|k_1)R(c_{f_1}+c_{g_1})\alpha\sum_{i=0}^{t}e^{-2\lambda(t-i)\beta}\nn\\
    &\quad+3(c_{f_2}^2+c_{g_2}^2+2\frac{1}{\lambda^2}(1+\gamma+\gamma R|\mca|k_1)^2(c_{f_1}^2+c_{g_1}^2))\sum^t_{i=0}e^{-2\lambda(t-i)\beta}\beta^2\nn\\
    &=\mathcal{O}\Bigg( e^{-2\lambda t\beta}+\tau\beta\sum^{\tau}_{i=0}e^{-2\lambda(t-i)\beta}+\tau\beta^2\sum^t_{i=1+\tau}e^{-2\lambda(t-i)\beta}+(\alpha+\beta^2)\sum^t_{i=0}e^{-2\lambda(t-i)\beta}\Bigg)\nn\\
    &=\mathcal{O}\Bigg( e^{-2\lambda t\beta}+\tau\beta e^{-2\lambda t\beta}\frac{1-e^{2\lambda \beta(\tau+1)}}{1-e^{2\lambda\beta}}+\tau\beta^2(e^{-2\lambda t\beta}-e^{-2\lambda \beta\tau})\frac{e^{2\lambda\beta(\tau+1)}}{1-e^{2\lambda\beta}}+(\alpha+\beta^2)\frac{e^{-2\lambda t\beta}-e^{2\lambda\beta}}{1-e^{2\lambda\beta}}\Bigg).
\end{align}
Similarly, for $t\leq \tau_{\beta_T}$, we obtain that
\begin{align}\label{eq:order2}
    \mathbb{E}[\|z_{t} \|^2]
    &\leq  \|z_0\|^2 e^{-2\lambda t\beta}+2(4Rc_{f_2}\beta+(a_{f_2}+a_{g_2})\tau_{\beta_T})\beta\sum_{i=0}^{t}e^{-2\lambda(t-i)\beta}\nn\\
    &\quad+\frac{4}{\lambda}(1+\gamma+\gamma R|\mca|k_1)R(c_{f_1}+c_{g_1})\alpha\sum_{i=0}^{t}e^{-2\lambda(t-i)\beta}\nn\\
    &\quad+3(c_{f_2}^2+c_{g_2}^2+2\frac{1}{\lambda^2}(1+\gamma+\gamma R|\mca|k_1)^2(c_{f_1}^2+c_{g_1}^2))\sum^t_{i=0}e^{-2\lambda(t-i)\beta}\beta^2\nn\\
    &=\mathcal{O}\bigg( e^{-2\lambda\beta t}+\tau\beta\sum^t_{i=0}e^{-2\lambda(t-i)\beta}\bigg)=\mathcal{O}\Bigg(e^{-2\lambda\beta t}+\tau\beta \frac{e^{-2\lambda\beta t}-e^{2\lambda\beta}}{1-e^{2\lambda\beta}}\Bigg).
\end{align}

We then bound $\sum^T_{t=0}\mathbb{E}[\|z_t\|^2]$. 
The sum is divided into two parts: $\sum^{\tau}_{t=0}\mathbb{E}[\|z_t\|^2]$ and $\sum^T_{t=\tau+1}\mathbb{E}[\|z_t\|^2]$, thus 

\begin{align}
    &\sum^{T}_{t=0}\mathbb{E}[\|z_t\|^2]\nn\\
    &=\sum^{\tau}_{t=0}\mathbb{E}[\|z_t\|^2]+\sum^T_{t=\tau+1}\mathbb{E}[\|z_t\|^2]\nn\\
    &=\sum^{\tau}_{t=0}(e^{-2\lambda\beta t}+\tau\beta \frac{e^{-2\lambda\beta t}-e^{2\lambda\beta}}{1-e^{2\lambda\beta}})+\sum^T_{t=\tau+1}\Bigg(e^{-2\lambda t\beta}+\tau\beta e^{-2\lambda t\beta}\frac{1-e^{2\lambda \beta(\tau+1)}}{1-e^{2\lambda\beta}}\nn\\
    &\quad+\tau\beta^2(e^{-2\lambda t\beta}-e^{-2\lambda \beta\tau})\frac{e^{2\lambda\beta(\tau+1)}}{1-e^{2\lambda\beta}}+(\alpha+\beta^2)\frac{e^{-2\lambda t\beta}-e^{2\lambda\beta}}{1-e^{2\lambda\beta}}\Bigg)\nn\\
    &=\frac{1-e^{-2\lambda\beta(T+1)}}{1-e^{-2\lambda\beta}}+\tau\beta\left((\tau+1)\frac{-e^{2\lambda\beta}}{1-e^{2\lambda\beta}}+\frac{1-e^{-2\lambda\beta (\tau+1)}}{(1-e^{2\lambda\beta})(1-e^{-2\lambda\beta})}\right)\nn\\
    &\quad+\tau\beta\frac{1-e^{2\lambda\beta(\tau+1)}}{1-e^{2\lambda\beta}}e^{-2\lambda\beta(\tau+1)}\frac{1-e^{-2\lambda\beta(T-\tau)}}{1-e^{-2\lambda\beta}}+\tau\beta^2\frac{e^{2\lambda\beta(\tau+1)}}{1-e^{2\lambda\beta}}\bigg(e^{-2\lambda\beta(\tau+1)}\frac{1-e^{-2\lambda\beta(T-\tau)}}{1-e^{-2\lambda\beta}}\nn\\
    &\quad-(T-\tau)e^{-2\lambda\beta\tau}\bigg)  +(\alpha+\beta^2)\frac{1}{1-e^{2\lambda\beta}}\Bigg(e^{-2\lambda\beta(\tau+1)}\frac{1-e^{-2\lambda\beta(T-\tau)}}{1-e^{-2\lambda\beta}}-(T-\tau)e^{2\lambda\beta}\Bigg)\nn\\
    &=\mathcal{O}\Bigg( \frac{1}{\beta}+\tau^2+\tau+\tau\beta T+\frac{\alpha+\beta^2}{\beta}T\Bigg).
\end{align}

Thus, we have that 
\begin{align}\label{eq:trackingorder}
    \frac{\sum^T_{t=0}\mathbb{E}[\|z_t\|^2]}{T}=\mathcal{O}\Bigg(\frac{1}{T^{1-b}}+\frac{(\log T)^2}{T}+\frac{\log T}{T^b}+\frac{1}{T^{a-b}}+\frac{1}{T^b}\Bigg)=\mathcal{O}\Bigg(\frac{\log T}{T^{\min \left\{a-b,b \right\}}}\Bigg).
\end{align}

We then plug the tracking error \eqref{eq:trackingorder} in \eqref{eq:main1}, and we have that
\begin{align}
    \frac{\sum^T_{t=0}\alpha\mathbb{E}[\|\nabla J(\theta_t)\|^2]}{\sum^T_{t=0}\alpha}=\mathcal{O}\Bigg(\frac{1}{T^{1-a}}\Bigg)+\mathcal{O}\Bigg(\frac{\log T}{T^{\min \left\{a-b,b \right\}}}\Bigg).
\end{align}

In the following we will recursively refine our bounds on the tracking error using the bound in \eqref{eq:main1}. 
Recall \eqref{eq:recursion0}, and denote $D=J(\theta_0-J^*)$, then 
\begin{align}\label{eq:recursion1}
    \frac{\sum^T_{t=0} \mathbb{E}[\|\nabla J(\theta_t)\|^2]}{T}&= \frac{D}{T\alpha}+\mathcal{O}\Bigg(\frac{\sum^T_{t=0} \sqrt{\mathbb{E}[\|\nabla J(\theta_t)\|^2]\mathbb{E}[\|z_t\|^2]}}{T}\Bigg)
    \nn\\
    &=\mathcal{O}\Bigg(\frac{1}{T\alpha}+\sqrt{\frac{\sum^T_{t=0}\mathbb{E}[\|\nabla J(\theta_t)\|^2]}{T}}\sqrt{\frac{\sum^T_{t=0}\mathbb{E}[\|z_t\|^2]}{T}}\Bigg).
\end{align}

In the first round, we upper bound $\frac{\sum^T_{t=0}\mathbb{E}[\|\nabla J(\theta_t)\|^2]}{T}$ by a constant. It then follows that 
\begin{align}\label{eq:103}
    \frac{\sum^T_{t=0}\mathbb{E}[\|\nabla J(\theta_t)\|^2]}{T}=\mathcal{O}\Bigg(\frac{1}{T^{1-a}}\Bigg)+\sqrt{\mathcal{O}\Bigg(\frac{\log T}{T^b}+\frac{1}{T^{a-b}}\Bigg)}=\mathcal{O}\Bigg(\frac{1}{T^{1-a}}\Bigg)+\mathcal{O}\Bigg(\frac{\sqrt{\log T}}{T^{\min\left\{ b/2, a/2-b/2\right\}}}\Bigg),
\end{align}
where we denote $\min\left\{ b/2, a/2-b/2 \right\}$ by $c/2$.
We then plug \eqref{eq:103} into \eqref{eq:recursion1}, and we obtain that
\begin{align}
     \frac{\sum^T_{t=0}\mathbb{E}[\|\nabla J(\theta_t)\|^2]}{T}= \mathcal{O}\Bigg(\frac{1}{T^{1-a}}\Bigg)+\mathcal{O}\Bigg(\frac{\sqrt{\log T}}{T^{c/2}}\sqrt{\frac{\sum^T_{t=0}\mathbb{E}[\|\nabla J(\theta_t)\|^2]}{T}}\Bigg).
\end{align}

\textbf{Case 1.}
If $1-a<c/2$, then bound in \eqref{eq:103} is $\mathcal{O}\Bigg(\frac{1}{T^{1-a}} \Bigg)$:
$
    \frac{\sum^T_{t=0}\mathbb{E}[\|\nabla J(\theta_t)\|^2]}{T}=\mathcal{O}\Bigg(\frac{1}{T^{1-a}}\Bigg).
$
Then 
\begin{align}
     \frac{\sum^T_{t=0}\mathbb{E}[\|\nabla J(\theta_t)\|^2]}{T} = \mathcal{O}\Bigg(\frac{1}{T^{1-a}}+\frac{\sqrt{\log T}}{T^{c/2}}\frac{1}{T^{1/2-a/2}}\Bigg).
\end{align}
Note that $c/2>1-a$, then $c/2+1/2-a/2>1-a$, thus the order would be 
\begin{align}
    \frac{\sum^T_{t=0}\mathbb{E}[\|\nabla J(\theta_t)\|^2]}{T}=\mathcal{O}\Bigg(\frac{1}{T^{1-a}}\Bigg).
\end{align}

Therefore, such a recursive refinement will not improve the convergence rate if $1-a < \frac{c}{2}$.

\textbf{Case 2.}
If $c>1-a\geq c/2$, then 
\begin{align}
    \frac{\sum^T_{t=0}\mathbb{E}[\|\nabla J(\theta_t)\|^2]}{T}=\mathcal{O}\Bigg(\frac{\sqrt{\log T}}{T^{c/2}}\Bigg).
\end{align}
Also plug this order in \eqref{eq:recursion1}, and we obtain that
\begin{align}
    &\frac{\sum^T_{t=0}\mathbb{E}[\|\nabla J(\theta_t)\|^2]}{T}=\mathcal{O}\Bigg(\frac{1}{T^{1-a}}\Bigg)+\mathcal{O}\Bigg(\frac{\sqrt{\log T}}{T^{c/2}}\frac{(\log T)^{1/4}}{T^{c/4}}\Bigg)=\mathcal{O}\Bigg(\frac{1}{T^{1-a}}+\frac{(\log T)^{\frac{3}{4}}}{T^{3c/4}}\Bigg).
\end{align}
Here, we start the second iteration. 
If $1-a \geq \frac{3c}{4}$, we know that the order is improved as follows
\begin{align}
    \frac{\sum^T_{t=0}\mathbb{E}[\|\nabla J(\theta_t)\|^2]}{T}=\mathcal{O}\Bigg(\frac{(\log T)^{\frac{3}{4}}}{T^{3c/4}}\Bigg).
\end{align}
And if $1-a<\frac{3c}{4}$, then order of \eqref{eq:103} will still be $\mathcal{O}\Bigg(\frac{1}{T^{1-a}}\Bigg)$. Thus we will stop the recursion, and we have that
\begin{align}
    \frac{\sum^T_{t=0}\mathbb{E}[\|\nabla J(\theta_t)\|^2]}{T}=\mathcal{O}\Bigg(\frac{1}{T^{1-a}}\Bigg).
\end{align}
This implies that if the recursion stops after some step until there is no further rate improvement, then the convergence rate will be $\mathcal{O}\Bigg(\frac{1}{T^{1-a}} \Bigg)$. Note in this case, since $1-a<c$, then there exists some integral $n$, such that $1-a<\frac{2^n-1}{2^n}c$, and after round $n$, the recursion will stop. Thus the final rate is $\mathcal{O}\Bigg(\frac{1}{T^{1-a}} \Bigg)$.

\textbf{Case 3.}
If $1-a\geq c$, then after a number of recursions, the order of the bound will be sufficiently close to $\mathcal O \left(\frac{\log T}{T^c}\right)$.

To conclude the three cases, when $1-a<c$, the recursion will stop after finite number of iterations, and the rate would be $\mathcal{O}\Bigg( \frac{1}{T^{1-a}}\Bigg)$; While when $1-a\geq c$, the recursion will always continue, and the fastest rate we can obtain is $\mathcal O \left(\frac{\log T}{T^c}\right)$. Thus the overall rate we can obtain can be written as 
\begin{flalign}
\mathcal{O}\Bigg( \frac{1}{T^{1-a}}+\frac{\log T}{T^c}\Bigg).
\end{flalign}

\subsection{Proof of Corollary \ref{col:1}}
We next look for suitable $a$ and $b$, such that the rate obtained is the fastest. It can be seen that the best rate is achieved when $1-a=c$, and at the same time $0.5<a\leq 1$ and $0<b<a$.
Thus, the best choices are $a=\frac{2}{3}$ and $b=\frac{1}{3}$, and the best rate we can obtain is \begin{align}
    \mathbb{E}[\|\nabla J(\theta_M)\|^2]=\frac{\sum^T_{t=0}\mathbb{E}[\|\nabla J(\theta_t)\|^2]}{T}=\mathcal{O}\Bigg(\frac{\log T}{T^{1-a}}\Bigg)=\mathcal{O}\Bigg(\frac{\log T}{T^{\frac{1}{3}}}\Bigg).
\end{align}



\section{Softmax Is Lipschitz and Smooth}
We first restate Lemma \ref{lemma:softmax_smooth} as follows, and then derive its proof.
\begin{Lemma}
The softmax policy is $2\sigma$-Lipschitz and $8\sigma^2$-smooth, i.e., for any $(s,a)\in\mcs\times\mca$, and for any $\theta_1,\theta_2\in\mathbb R^N$, $|\pi_{\theta_1}(a|s)-\pi_{\theta_2}(a|s)| \leq 2\sigma \|\theta_1-\theta_2 \|$ and $\|\nabla\pi_{\theta_1}(a|s)-\nabla \pi_{\theta_2}(a|s)  \|\leq 8\sigma^2 \|\theta_1-\theta_2 \|$.
\end{Lemma}


\begin{proof}
By the definition of the softmax policy,  for any $a\in\mca$, $s\in \mcs$ and $\theta\in \mathbb R^N$,
\begin{align}
    \pi_{\theta}(a|s)=\frac{e^{\sigma {\theta}^\top \phi_{s,a}}}{\sum_{a' \in \mathcal{A}}e^{\sigma {\theta}^\top  \phi_{s,a'}}},
\end{align} where $\sigma>0$ is a constant. Then, it can be shown that 
\begin{align}
    \nabla \pi_{\theta}(a|s)
    &=\frac{1}{\left(\sum_{a' \in \mathcal{A}}e^{\sigma{\theta}^\top  \phi_{s,a'}}\right)^2} 
    \left(\sigma e^{\sigma{\theta}^\top   \phi_{s,a}}\phi_{s,a}\left(\sum_{a' \in \mathcal{A}}e^{\sigma{\theta}^\top  \phi_{s,a'}}\right)-\left(\sum_{a' \in \mathcal{A}}\sigma e^{\sigma{\theta}^\top  \phi_{s,a'}}\phi_{s,a'}\right)e^{\sigma{\theta}^\top   \phi_{s,a}}\right)\nn\\
    &=\frac{\sigma}{(\sum_{a' \in \mathcal{A}}e^{\sigma{\theta}^\top  \phi_{s,a'}})^2}  \left(\sum_{a' \in \mathcal{A}} \phi_{s,a} e^{\sigma{\theta}^\top  (\phi_{s,a}+\phi_{s,a'})}-\phi_{s,a'}e^{\sigma{\theta}^\top  (\phi_{s,a}+\phi_{s,a'})}\right)\nn\\
    &=\frac{\sigma\sum_{a' \in \mathcal{A}}(\phi_{s,a}-\phi_{s,a'})e^{\sigma{\theta}^\top  (\phi_{s,a}+\phi_{s,a'})}}{(\sum_{a' \in \mathcal{A}}e^{\sigma{\theta}^\top  \phi_{s,a'}})^2}.
\end{align}
Thus,
 \begin{align}
    ||\nabla \pi_{\theta}(a|s)|| \leq 2\sigma \frac{\sum_{a' \in \mathcal{A}}e^{\sigma{\theta}^\top  (\phi_{s,a}+\phi_{s,a'})}}{\left(\sum_{a' \in \mathcal{A}}e^{\sigma{\theta}^\top  \phi_{s,a'}}\right)^2}=2\sigma \frac{e^{\sigma{\theta}^\top  \phi_{s,a}}}{\sum_{a' \in \mathcal{A}}e^{\sigma{\theta}^\top  \phi_{s,a'}}}
    \leq 2\sigma,
\end{align}
 where the last step is due to the fact that $\frac{e^{\sigma{\theta}^\top  \phi_{s,a}}}{\sum_{a' \in \mathcal{A}}e^{\sigma{\theta}^\top  \phi_{s,a'}}} \leq1$.
 
Note that for any $\theta_1$ and $\theta_2$, there exists some $\alpha \in (0,1)$ and $\bar \theta=\alpha\theta_1+(1-\alpha)\theta_2$, such that 
\begin{align}
    \|\nabla \pi_{\theta_1}(a|s)-\nabla \pi_{\theta_2}(a|s) \| \leq \|\nabla^2 \pi_{\bar \theta}(a|s)\|\times\|\theta_1-\theta_2 \|.
\end{align}
Here, $\nabla^2 \pi_{\theta}(a|s)$ denotes the Hessian matrix of $\pi_{\theta}(a|s)$ at $\theta$.
Thus it suffices to find an universal bound of $\|\nabla^2 \pi_{\theta}(a|s) \|$ for any $\theta$ and $(a,s) \in \mathcal{A}\times\mathcal{S}$.

Note that $\nabla \pi_{\theta}(a|s)=\frac{\sigma\sum_{a' \in \mathcal{A}}(\phi_{s,a}-\phi_{s,a'})e^{\sigma{\theta}^\top  (\phi_{s,a}+\phi_{s,a'})}}{\left(\sum_{a' \in \mathcal{A}}e^{\sigma{\theta}^\top  \phi_{s,a'}}\right)^2}$ is a sum of vectors $(\phi_{s,a}-\phi_{s,a'})$ with each entry multiplied by $\frac{\sigma e^{\sigma \theta^\top (\phi_{s,a}+\phi_{s,a'})}}{\left(\sum_{a' \in \mathcal{A}}e^{\sigma{\theta}^\top  \phi_{s,a'}}\right)^2}$.
Then it follows that
\begin{flalign}
\nabla^2 \pi_{\theta}(a|s) = \sigma\sum_{a'\in \mathcal{A}} (\phi_{s,a}-\phi_{s,a'})  \left(\nabla \frac{e^{\sigma{\theta}^\top  (\phi_{s,a}+\phi_{s,a'})}}{\left(\sum_{a' \in \mathcal{A}}e^{\sigma{\theta}^\top  \phi_{s,a'}}\right)^2}\right)^\top.
\end{flalign}
Thus, to bound $\|\nabla^2 \pi_{\theta}(a|s) \|$, we compute the following:
\begin{align}\label{eq:37}
    &\nabla \frac{e^{\sigma{\theta}^\top  (\phi_{s,a}+\phi_{s,a'})}}{(\sum_{a' \in \mathcal{A}}e^{\sigma{\theta}^\top  \phi_{s,a'}})^2}\nn\\
    &=\sigma \frac{e^{\sigma{\theta}^\top  (\phi_{s,a}+\phi_{s,a'})}\left((\sum_{a' \in \mathcal{A}}e^{\sigma{\theta}^\top  \phi_{s,a'}})(\phi_{s,a}+\phi_{s,a'})-2(\sum_{a' \in \mathcal{A}}e^{\sigma{\theta}^\top  \phi_{s,a'}}\phi_{s,a'})\right)}
    {(\sum_{a' \in \mathcal{A}}e^{\sigma{\theta}^\top  \phi_{s,a'}})^3}.
\end{align}
Then the norm of \eqref{eq:37} can be bounded as follows:
\begin{align}
    &\left\|\nabla \left(\frac{e^{\sigma{\theta}^\top  (\phi_{s,a}+\phi_{s,a'})}}{(\sum_{a' \in \mathcal{A}}e^{\sigma{\theta}^\top  \phi_{s,a'}})^2}\right)\right\|\nn\\
    &\leq \sigma\frac{2e^{\sigma{\theta}^\top  (\phi_{s,a}+\phi_{s,a'})}\left( \sum_{a' \in \mathcal{A}}e^{\sigma{\theta}^\top  \phi_{s,a'}} +(\sum_{a' \in \mathcal{A}}e^{\sigma{\theta}^\top  \phi_{s,a'}})\right)}{(\sum_{a' \in \mathcal{A}}e^{{\theta}^\top  \phi_{s,a'}})^3}\nn\\
    &=4\sigma\frac{e^{\sigma{\theta}^\top  (\phi_{s,a}+\phi_{s,a'})}}{\left(\sum_{a' \in \mathcal{A}}e^{\sigma{\theta}^\top  \phi_{s,a'}}\right)^2}\nn\\
    &\leq 4\sigma.
\end{align}

Plug this in the expression of $\nabla^2 \pi_{\theta}(a|s)$, we obtain that 
\begin{align}
    ||\nabla^2\pi_{\theta}(a|s)|| \leq 8\sigma^2.
\end{align}

Thus the softmax policy is $2\sigma$-Lipschitz and $8\sigma^2$-smooth. This completes the proof.
\end{proof}

\end{document}